\documentclass[preprint,12pt]{colt2025} 

\newcommand\blfootnote[1]{%
  \begingroup
  \renewcommand\thefootnote{}\footnote{#1}%
  \addtocounter{footnote}{-1}%
  \endgroup
}

\usepackage[font=small,labelfont=bf]{caption} 
\usepackage{titlesec}
\titleformat{\subsubsection}{\bfseries\normalsize}{\thesubsubsection}{1em}{}

\usepackage{enumitem}
\setlist[itemize]{itemsep=0pt, topsep=0pt, partopsep=0pt}

\usepackage{algorithm}
\usepackage{algorithmic}

\usepackage{longtable}

\def\Ev{\mathbb{E}}
\def\C{\mathbb{C}}
\def\supp{supp}
\def\g{\mathfrak{g}}
\def\sg{\mathfrak{g}_\Phi}

\usepackage{bbm}
\def\ind{\mathbbm{1}}
\def\kr{\mathcal{K}}

\DeclareMathOperator*{\argmin}{arg\,min}
\def\one{\mathbf{1}}
\title{How to safely discard features based on aggregate SHAP values}

\usepackage{times}
\usepackage{enumitem}
\usepackage{natbib}

\newtheorem{appxlemma}[theorem]{Lemma}
\newtheorem{appdfn}[theorem]{Definition}
\newtheorem{apptheorem}[theorem]{Theorem}

\coltauthor{%
 \Name{Robi Bhattacharjee*} \Email{robi.bhattacharjee@wsii.uni-tuebingen.de}\\
 \addr University of Tübingen and Tübingen AI Center
 \AND
 \Name{Karolin Frohnapfel*} 
 \Email{karolin.frohnapfel@uni-tuebingen.de}\\
 \addr University of Tübingen%
 \AND
 \Name{Ulrike {von Luxburg}} \Email{ulrike.luxburg@uni-tuebingen.de}\\
 \addr University of Tübingen and Tübingen AI Center
}

\begin{document}

\blfootnote{*Equal contribution. Preprint.}

\maketitle

\begin{abstract}%

SHAP is one of the most popular \textit{local} feature-attribution methods. Given a function $f$ and an input $x \in \R^d$, it quantifies each feature's contribution to $f(x)$. Recently, SHAP has been increasingly used for \textit{global} insights: practitioners average the absolute SHAP values over many data points to compute global feature importance scores, which are then used to discard ``unimportant'' features.
In this work, we investigate the soundness of this practice by asking whether small aggregate SHAP values necessarily imply that the corresponding feature does not affect the function. Unfortunately, the answer is no: even if the $i$-th SHAP value equals $0$ on the entire data support, there exist functions that clearly depend on Feature $i$. The issue is that computing SHAP values involves evaluating $f$ on points outside of the data support, where $f$ can be strategically designed to mask its dependence on Feature $i$.
To address this, we propose to aggregate SHAP values over the \textit{extended} support, which is the product of the marginals of the underlying distribution. With this modification, we show that a small aggregate SHAP value implies that we can safely discard the corresponding feature.
We then extend our results to KernelSHAP, the most popular method to approximate SHAP values in practice. We show that if KernelSHAP is computed over the extended distribution, a small aggregate KernelSHAP value justifies feature removal. This result holds independently of whether KernelSHAP accurately approximates true SHAP values, making it one of the first theoretical results to characterize the KernelSHAP algorithm itself.
Our findings have both theoretical and practical implications. We introduce the “Shapley Lie algebra”, which offers algebraic insights that may enable a deeper investigation of SHAP and we show that a simple preprocessing step -- randomly permuting each column of the data matrix -- enables safely discarding features based on aggregate SHAP and KernelSHAP values.

\end{abstract}

\begin{keywords}%
  Interpretability, Explainable machine learning, XAI, Shapley values, feature selection, Lie Theory%
\end{keywords}

\section{Introduction}

\begin{figure}[tb]
    \centering
    \includegraphics[width=\textwidth]{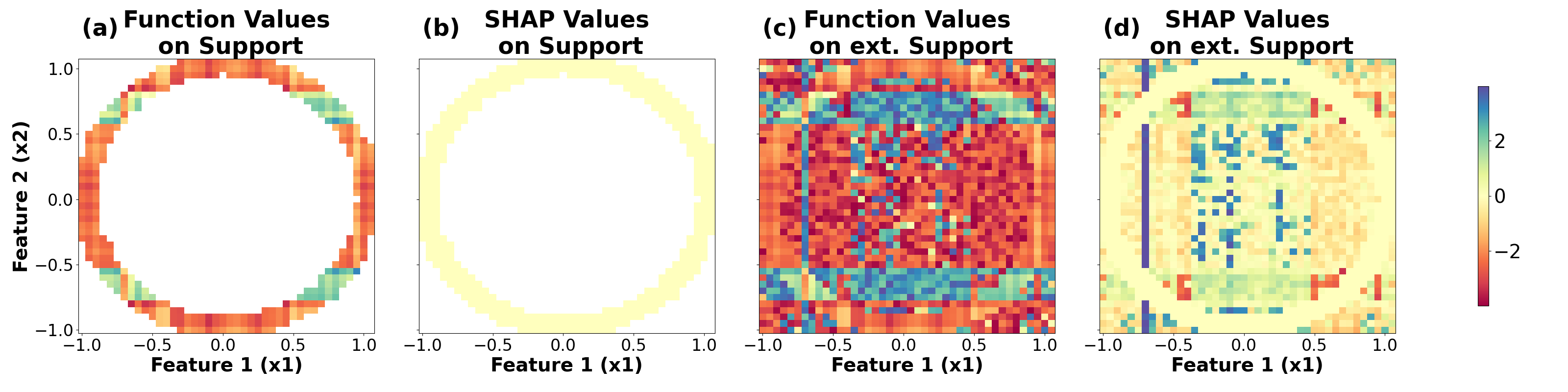}  
    \caption{\textbf{Example of a function where the aggregate SHAP value of Feature $1$ is $0$, yet the function depends on this feature.} \textbf{(a):} Function $f:\R^2 \to \R$, supported on a ring with the color depicting the function value. The function clearly depends on both Features $1$ and $2$. \textbf{(b):} Point-wise SHAP values $\phi_1(\mu, f,x)$ of Feature $1$ are constantly $0$ on the support. This provides the counter-example we have been looking for. \textbf{(c) and (d):} Function and SHAP values on the extended support. Here the SHAP values are not constantly $0$ any more, illustrating the direction towards resolving the issue of the counter-example. 
    }
    \label{fig:counterexample} 
\end{figure}

Due to the widespread adoption of large, opaque models, explainability has become an essential topic in machine learning. One particularly prominent application domain is \textit{scientific discovery}, where practitioners train a model not only for accurate predictions but also to gain insight into their specific problem and its underlying mechanisms. In such cases, the true value of the machine learning model lies in the understanding it provides, making interpretability techniques critical. 

In science, SHAP~\citep{LundbergLee2017} is by far the most widely used method for generating explanations. It is a local feature-attribution method that is applied across various fields including biology~\citep{Appl_SHAP:Berdugo2022}, geoscience~\citep{Appl_SHAP:Jiang2024}, medicine~\citep{Appl_SHAP:Martinez-Ruiz2023}, psychiatry~\citep{Appl_SHAP:Giuntella2021}, physics~\citep{Appl_SHAP:Li2021}, and chemistry~\citep{wojtuch2021can}. SHAP operates as follows: for an input distribution $\mu$ over $\R^d$, a model $f: \R^d \to \R$, and a fixed input point $x \in \R^d$, it outputs $d$ values, $\phi_1(\mu, f, x), \dots, \phi_d(\mu, f, x)$ that quantify the ``impact" that each feature had in predicting $f(x)$. There are a variety of ways these values can be defined and implemented. In this work, we exclusively focus on \textit{interventional SHAP}, which is the more widely used definition, and on \textit{KernelSHAP}~\citep{LundbergLee2017}, which is the most popular algorithm for approximating it. 

When it was invented, SHAP was clearly designed towards \textit{local} explanations, which apply to a specific input point. However, especially in scientific contexts, SHAP has recently become 
popular for providing \textit{global} feature importance, for example in \citet{Appl_MAS:Greenwood2024}, \citet{Appl_MAS:SharmaTimilsina2024}, \cite{Appl_MAS:Bernard2023}, \citet{Appl_MAS:Delavaux2023}, \citet{Appl_MAS:Chen2022}, \citet{Appl_MAS:Ekanayake2022}, \citet{Appl_MAS:Qiu2022}, \citet{Appl_MAS:Rane2022}, \citet{Appl_MAS:Wang2022} or \citet{Appl_MAS:Yang2022}, see also Appendix~\ref{app:use_cases}. Practitioners average the \textit{absolute} SHAP values 
$|\phi_i(\mu, f, x)|$ over many data points drawn from the input distribution $\mu$ to obtain $d$ \textit{aggregate SHAP values} $\overline{\phi_1}(\mu, f), \dots, \overline{\phi_d}(\mu, f)$. Features are then typically sorted, selected, and globally interpreted based on these aggregate SHAP values.
Despite its popularity, this method lacks any theoretical guarantees. In this work, we address this gap by studying a fundamental property we call \textit{soundness}, which means that features with a low global importance score are not impactful on making predictions. More specifically, we seek to answer the following question:
\begin{center}
\textit{If $\overline{\phi_i}(\mu, f)$ is small, does that mean Feature $i$ is globally irrelevant for predictions made by $f$?}
\end{center}

\subsection{Our Contributions}

In Section \ref{sec:counter_example} we begin with the extreme version of this problem where $\overline{\phi_i}(\mu, f)$ precisely equals~$0$. Our question then becomes the following: if $\overline{\phi_i}(\mu, f) = 0$, can $f$ be computed without \textit{any} access to $x_i$?  
Unfortunately, the answer to this question is no. As shown in Panels (a) and (b) of Figure \ref{fig:counterexample}, there exist examples where the SHAP value $\phi_i(\mu, f, x)$ is constantly $0$ across $\supp(\mu)$ (Panel (b)) and yet $f$ clearly exhibits variation across $x_i$ within $\supp(\mu)$ (Panel (a)). The core issue is that computing $\phi_i(\mu, f, x)$ requires evaluating $f$ on points that are potentially \textit{outside} $\supp(\mu)$, and those values can be strategically chosen to cause $\phi_i(\mu, f, x)$ to equal $0$ \textit{inside} $\supp(\mu)$ (see Figure~\ref{fig:counterexample}). 

In Section \ref{sec:main_result_one}, we expand on this observation and introduce the notion of the \textit{extended support} $supp(\mu^*)$, which is the product of the supports of the marginal distributions $\mu_i$ of $\mu$ across each feature. Then, in our first result (Theorem~\ref{main_theorem}), we show that constant-zero aggregate SHAP values over the extended support $supp(\mu^*)$ \textit{is} sufficient for discarding a feature. This is illustrated in Panel~(d) of Figure \ref{fig:counterexample} where the dependence of $f$ on $x_1$ becomes apparent when looking at SHAP values over $supp(\mu^*)$. 

In Section \ref{sec:robust_theorem}, we extend our result from the extreme case where $\overline{\phi}_i(\mu, f) = 0$ to cases where this approximately holds. To incorporate the entire extended support, we propose aggregating and computing SHAP values fully using the extended distribution $\mu^*$. We then show (Theorem~\ref{theorem:robust_distribution_bound}) that doing so gives a more robust bound on the impact of removing a feature with a small aggregate SHAP value $\overline{\phi}_i(\mu^*, f)$. 

In Section \ref{section:kernel_shap} we turn our attention to the finite sample regime where SHAP values are computed with respect to $X \sim \mu^n$. This setting poses an additional challenge: there are no known results bounding how well the most popular algorithm for computing SHAP values, KernelSHAP, approximates the true SHAP values. Surprisingly, our techniques completely circumvent this by proving the first known soundness result \textit{directly} about KernelSHAP. In Theorem \ref{thm:kernel_shap}, we show that when KernelSHAP values are computed over a data sample from the extended distribution $\mu^*$, a small aggregate value implies that the feature has a small impact on the prediction. Furthermore, we observe that sampling from $\mu^*$ can be easily implemented in practice -- simply permute each feature column of a data matrix that is sampled from the original distribution $\mu$ (lines 1-3 of Algorithm \ref{alg:SHAP_extended_support}). With this simple modification, our theorem \textit{directly applies} to real-life implementations of KernelSHAP.

\begin{algorithm}[t]
\caption{Sound Aggregate KernelSHAP}\label{alg:SHAP_extended_support}
\raggedright
\textbf{Input:} $X \in \R^{n\times d}$: data matrix; $f$: function; $i\in [d]$: feature of interest
\vspace{-0.5cm}
\begin{algorithmic}[1]
\FOR{$j \in [d]$}
    \STATE \textcolor{red}{Randomly permute $j$-th column: $\Tilde{X}_j \gets \textrm{Permute}(X_j)$}
\ENDFOR
\STATE Save shuffled data matrix: $\Tilde{X} \gets (\Tilde{X}_1, \dots, \Tilde{X}_d)$
\STATE Calculate local SHAP values: $\phi_i^{(1)}, \dots, \phi_i^{(n)} \gets \textrm{KernelSHAP}(\Tilde{X}, f)$
\STATE Aggregate: $\overline{\phi_i} \gets \frac{1}{n} \sum_{k=1}^n \rvert \phi_i^{(k)} \lvert$ 
\end{algorithmic}
\end{algorithm}

We believe our work has interesting implications both in theory and in practice. From the theoretical side, our main contributions are: 
\begin{itemize}[itemsep=0pt, topsep=0pt]
\item Two theorems (Theorems \ref{main_theorem} and \ref{theorem:robust_distribution_bound}) characterizing when we can safely discard features using aggregate SHAP values. 
\item The first soundness analysis of KernelSHAP that holds \textit{independently} of how well KernelSHAP approximates the true SHAP values. 
\item A novel technical tool we call the Shapley Lie algebra (Definition \ref{defn:shapley_algebra}). This construction captures many useful algebraic properties of SHAP values which are central to proving  our main results. We believe our techniques might be useful to studying other properties of SHAP and KernelSHAP as well. 
\end{itemize}
We also note that our work is \textit{not} intended to exclusively provide novel algorithms for feature selection: there exist other approaches for doing so. Instead, \textit{this work characterizes the soundness of an approach that is already widely used in practice}. Our idea of using the extended distribution $\mu^*$ is intended to provide a theoretically justified modification of SHAP that enjoys provable soundness while (hopefully) preserving other desirable aspects of the algorithm. 

For practitioners, our work has a very clear and simple implication: to discard features based on aggregate SHAP values, it is not enough to average SHAP values over a (subset of) the original data points. Instead, one has to sample from the extended support, which luckily is very simple by randomizing features, as can be seen in the pseudo-code above (Algorithm~\ref{alg:SHAP_extended_support}). 

\subsection{Related Work}
\paragraph{SHAP values:} 
Since their introduction to the machine learning community by \citet{LundbergLee2017} SHAP values, which originate from game theory \citep{Shapley1953}, have gained increasing attention. 
See for example \citet{Lundberg2020}, \citet{Covert2020}, \citet{Frye2020} and \citet{Bordt2023}, just to state a few. Additionally to SHAP values we also look into KernelSHAP (\citealp{LundbergLee2017}, \citealp{CovertLee21})
, which is the most widely used approximation algorithm for SHAP values.
%
Similar to us, \citet{Slack2020} exploit the fact that interventional SHAP values are calculated using data points outside the distribution, however, with the different goal of masking an unfair algorithm as fair.
Also \citet{Merrick2020} and \citet{Kumar2020} consider this when investigating the axioms of SHAP values.

\paragraph{Global Feature Importance:}
In explainable machine learning many global feature importance methods exist, such as LOCO \citep{Lei2018_LOCO} or SAGE \citep{Covert2020}. 
However, in practice scientists also tend to simply aggregate local feature attributions to get global insights. 
While the focus of this paper is on the aggregation of SHAP values, aggregating other feature attribution methods such as LIME \citep{VanDerLinden2019} and Anchors \citep{Mor2024} has been proposed as well.
The idea of using explainability techniques for feature importance has become a subject of ongoing research (\citealp{Hooker2019}, \citealp{Merrick2020}, \citealp{Kumar2020}, \citealp{Ewald2024}, \citealp{Verdinelli2024}) and some also investigate the possibility of performing feature selection \citep{Marcilio2020} or data selection \citep{Wang2024}.

\paragraph{Explainable Machine Learning:}
While computing SHAP values is one of the most widely used method of feature attribution in explainable machine learning (see \citet{molnar2022} for an overview), many other exist, such as LIME \citep{Ribeiro2016_LIME}, Integrated Gradients \citep{Sundararajan2017_IG} and Anchors \citep{Ribeiro2018_Anchors}.
The literature on these methods is vast with a lot of work dedicated on giving theoretical guarantees for feature attribution methods. See for example \citet{Dasgupta2022}, \citet{Bilodeau2024} and \citet{Bressan2024}.

\section{Preliminaries}

\subsection{Notation}

We consider explanations for functions $f: \R^d \to \R$. For $x \in \R^d$, we let $(x_1, \dots, x_d)$ denote its coordinates. For a subset of indices $S \subseteq [d] = \{1, \dots, d\}$, we let $S^c$ denote its complement and $\R^S$ denote the projection of $\R^d$ onto its coordinates in $S$. That is, for $x \in \R^d$ we let $x_S = (x_i: i \in S) \in \R^S$. 
It will be useful to apply functions whose coordinates are drawn from different points. To denote this, if $x^{(1)}, \dots, x^{(k)} \in \R^d$ are $k$ points and $S^{(1)}, \dots, S^{(k)}$ are $k$ disjoint subsets that partition $[d]$, then we let  $f\left(x^{(1)}_{S^{(1)}}, \dots, x^{(k)}_{S^{(k)}}\right)=f(x)$ where $x$ is the unique point such that $x_{S^{(i)}} = x^{(i)}_{S^{(i)}}$. 
%

\subsection{The SHAP explanation method}

SHAP is a local posthoc explanation method that generates a separate explanation for each individual prediction. Given a data distribution $\mu$, a function $f$, and a data point $x$, it calculates $d$ SHAP values which quantify the contribution of each feature to the output of $f$ at $x$. 
To do so, it makes use of a \textit{value function} $v_S(\mu, f, x)$ that associates each subset $S \subseteq [d]$ of features with the prediction $f(x)$. $v_S$ is intended to simulate the behavior that $f$ might have if it only had access to features inside $S$. The current literature typically considers two main choices of value functions.

\begin{definition}[Value Function]\label{defn:value_function}
Let $\mu$ be a distribution over $\R^d$, $f: \R^d \to \R$ be a function, and $x \in \R^d$ a point. Let $S \subseteq [d]$ be a subset of features. The observational and interventional value functions corresponding to $S$ are defined as
\begin{align*}
    v_S^{obs}(\mu, f, x) &= \mathbb{E}_{X \sim \mu}[f(X)|X_S = x_S],\\
    v_S^{int}(\mu, f, x) &= \mathbb{E}_{X \sim \mu}[f(x_S, X_{S^c})].
\end{align*}
\end{definition}

$v_S^{obs}$ represents the value of $f$ when feature values outside of $S$ are sampled from the conditional distribution, while $v_S^{int}$ does the same using the marginal distribution. Due to the difficulty of sampling from a conditional distribution, the interventional value function is more widely used in practice, and from this point forward we will exclusively use it. To simplify notation, we will simply write $v_S$ to mean $v_S^{int}$. 


\begin{definition}[SHAP Values]\label{defn:shap_value}
Let $\mu$ be a distribution over $\R^d$, $f: \R^d \to \R$ be a function, and $x \in \R^d$ a point. For $1 \leq i \leq d$, the $i$th SHAP value of $f$ at $x$ is defined as $$\phi_i(\mu, f, x) = \frac{1}{d}\sum_{S \subseteq [d] \setminus \{i\}} \binom{d-1}{|S|}^{-1} \left(v_{S \cup \{i\}}(\mu, f, x) - v_S(\mu, f, x)\right).$$ 
\end{definition}

SHAP values are designed to effectively distill the information provided by the value function over all $2^d$ possible subsets of features into one value $\phi_i$ per feature.

Although SHAP values are primarily a local explanation method, they are increasingly used to derive global, feature-based explanations for machine learning models. This is typically achieved by averaging the absolute values of SHAP values for a given feature across the entire data distribution. 

\begin{definition}[Aggregate SHAP Values]
Let $\mu, f$ be a distribution and a function. 
Then the aggregate SHAP values $\overline{\phi_i}(\mu, f)$ are defined as $\overline{\phi_i}(\mu, f) = \Ev_{x \sim \mu} \left| \phi_i(\mu, f, x) \right|.$
\end{definition}

Practitioners typically interpret these values by discarding features with small aggregate SHAP values and concentrating on those with relatively large ones. The main purpose of this paper is to investigate how sound this practice is. We now formalize what it means to be able to safely ``discard" a feature. 

\begin{definition}[Determined Function / Discarding Features]\label{defn:det_func}
Let $S \subseteq [d]$ be a set of indices, \linebreak $f: \R^d \to \R$ a function, and $\mathcal{X} \subseteq \R^d$ a subset. $f$ is $S$-determined over $\mathcal{X}$ if for all $a, b \in \mathcal{X}$, $a_S = b_S \implies f(a) = f(b).$ We say that Feature $i$ can be \textit{discarded} for function $f$ over $\mathcal{X}$ if $f$ is $[d] \setminus \{i\}$-determined over $\mathcal{X}$. 
Additionally, we set the convention that a $\emptyset$-determined function is a constant function. 
\end{definition}

Intuitively, we can discard a feature if it does not ``influence'' the outcome of the function on the data support. 

\section{Characterization of Aggregate SHAP Values}\label{sec:aggregate_characterization}

The main question of this paper is to investigate whether features that have small aggregate SHAP values can be safely discarded or not. Surprisingly, and opposed to current practice in data science, the answer to this question is no. Let us show a simple counter-example. 

\subsection{Constant-zero SHAP values on the entire support do not allow to discard features}\label{sec:counter_example}



We consider a data-generating distribution $\mu$ with support on a two-dimensional ring, and a function~$f$ that is defined on this support
(Figure~\ref{fig:counterexample}, Panel (a)). The distribution $\mu$ and the function $f$ are chosen in such a way that the pointwise SHAP values $\phi_1(\mu,f,x)$ of Feature $1$ are constantly~$0$ on the support of $\mu$, so in particular the aggregate SHAP value 
$\overline{\phi_1}(\mu, f)$ 
is $0$. Yet, the function $f$ obviously is not independent of Feature $1$, hence this feature cannot be discarded. 
The key to achieving this behavior is the fact that we use interventional SHAP: to compute the value functions, we sample from the marginal distributions of both features, whose supports extend beyond the ring. In our example, this allowed us to strategically choose the function values of $f$ on this ``extended support'' such that the SHAP values \textit{within the support} are constantly $0$. Observe that the SHAP values in these out-of-support regions are no longer $0$  (Figure~\ref{fig:counterexample}, Panel (d)). To construct this example, we used a linear program. See Appendix~\ref{app:lin_prog_counterexample} for details and more examples with a similar behavior.


\subsection{Constant-zero SHAP values over the extended support allow to discard features}\label{sec:main_result_one}

Contemplating the counter-example above leads to the following idea: to understand whether we can discard Feature~$i$, we need to look at its SHAP values \textit{beyond} the support of the data distribution. To this end, we begin by defining a natural distribution associated with $\mu$ that characterizes precisely where we look. 
This ``extended distribution'' is constructed to have each feature independently range over its entire support according to $\mu$. It is formally defined as follows:

\begin{definition}[Extended distribution and extended support]\label{defn:extended_distribution}
    Let $\mu$ be a distribution over $\R^d$, and let $\mu_i$ denote its marginal distribution of Feature $i$. Let $\mu_1^*, \mu_2^*, \dots, \mu_d^*$ denote independent distributions such that $\mu_i^*$ is identically distributed to $\mu_i$. Then the extended distribution $\mu^*$ is defined as the product distribution $\mu^* = \prod_{i= 1}^d \mu_i^*$. We call the support of the extended distribution $supp(\mu^*)$ the extended support.
\end{definition}

 


The obvious question is now whether a constant-zero SHAP value across the \textit{extended} support implies that a feature can be safely discarded? We answer this affirmatively in the following theorem, which is the first main result of this paper. 

\begin{theorem}[Discarding features based on constant-zero SHAP values on extended support]\label{main_theorem}\\
Let $\mu$ be a distribution on $\R^d$ and $f: \R^d \to \R$ a measurable function. 
Let $1 \leq i \leq d$ be a feature. Then $f$ is $[d] \setminus \{i\}$-determined over $supp(\mu^*)$ if and only if $\phi_i(\mu, f, x) = 0$ for all $x \in supp(\mu^*)$.
\end{theorem}

Note that this theorem concerns SHAP values $\phi_i(\mu, f, x)$ with respect to the original distribution $\mu$, but we need to consider these values for all points $x$ in the extended support $supp(\mu^*)$. More generally, the theorem would equally hold if we considered the SHAP values $\phi_i(\mu^*, f, x)$ instead. 
%
%
Additionally, this theorem implies that when $\mu$ has full support on $\R^d$, a constant-zero SHAP value is a sufficient condition for discarding a feature. 

\subsection{Proof of Theorem \ref{main_theorem}}

Let $F$ denote the vector space of all measurable functions from $\supp(\mu^*) \to \C$ (we generalize to functions ranging over complex values for technical reasons based on the nicer properties of complex vector spaces). 
To prove Theorem~\ref{main_theorem}, we begin by reframing it as a statement about linear operators that act on $F$. To do so, we use the following definitions. 

\begin{definition}[Determined Function Space]\label{defn:det_func_vec_space}
Let $F_S$ denote the vector space of all measurable $S$-determined functions from $\supp(\mu^*) \to \C$. 
\end{definition}

Observe that $F_S$ is a vector space because linear combinations of $S$-determined functions are $S$-determined themselves. Next, we show (proof in Appendix \ref{proof:defn:value_operator}) that there exist linear operators $F \to F$ that correspond to value functions (Definition~\ref{defn:value_function}) and SHAP values (Definition~\ref{defn:shap_value}). 

\begin{lemma}[Value operator]\label{defn:value_operator} Let $S \subseteq [d]$. There exists a linear operator $\upsilon_S: F \to F$ such that $(\upsilon_S f) (x) = v_S(\mu, f, x)$ for all $f \in F$. 
\end{lemma}

\begin{definition}[SHAP operator]\label{defn:shapley_operator}
Let $1 \leq i \leq d$ be a feature. Then the SHAP operator ${\Phi_i: F \to F}$ is defined as $\Phi_i f = A_if - B_if$ where $$A_if = \frac{1}{d}\sum_{S \subseteq [d] \setminus \{i\}} \binom{d-1}{|S|}^{-1} \upsilon_{S \cup \{i\}} f \text{ and }B_if = \frac{1}{d}\sum_{S \subseteq [d] \setminus \{i\}} \binom{d-1}{|S|}^{-1} \upsilon_{S} f.$$ 
\end{definition}

Expanding out this definition immediately implies that $(\Phi_if)(x)$ is precisely the $i$th SHAP value $\phi_i(\mu, f, x)$ (Definition~\ref{defn:shap_value}) of $f$ at $x$ with respect to $\mu$. Thus we can reframe the statement of Theorem~\ref{main_theorem} as follows: for any $f \in F$, $\Phi_i f = 0 \iff f \in F_{[d] \setminus \{i\}}.$
Our strategy to prove this will be to derive a useful set of properties of value operators that culminate in the following characterizations of $A_i$ and $B_i$. 
\begin{lemma}[Properties of $A_i$ and $B_i$]\label{lemma:key_properties}
The operators $A_i$ and $B_i$ satisfy the following properties:
\vspace{-\topsep} 
\begin{enumerate}[noitemsep] 
    \item \textbf{Image of $A_i$:} For all $S \subseteq [d]$, $A_i(F_S) \subseteq F_S$.
    \item \textbf{Image of $B_i$:} $B_i(F) \subseteq F_{[d] \setminus \{i\}}$.
    \item \textbf{Kernel of $A_i$:} $A_i^{-1}(\{0\}) = \{0\}$.
\end{enumerate}
\end{lemma}
Properties 1 and 2 demonstrate that $A_i$ and $B_i$ both tend to preserve determined functions, while Property 3 implies that $A_i$ has a trivial kernel.
Lemma~\ref{lemma:key_properties} is a consequence of the algebraic structure of the value operators, and its proof is surprisingly involved. As we will see, the set of value operators $\{\upsilon_S: S \subseteq [d]\}$ forms a solvable Lie algebra, which provides useful structure to prove the lemma. We defer the proof of Lemma~\ref{lemma:key_properties} to Section~\ref{sec:lie_algebra}. 
Instead, let us show how this lemma implies Theorem~\ref{main_theorem}.\\

\begin{proof}[Theorem~\ref{main_theorem} (Sketch); full proof in Appendix~\ref{app:main_theorem_proof}] According to our previous discussion, it suffices to show that for all $f \in F$ and $1 \leq i \leq d$, $\Phi_if = 0$ if and only if $f \in F_{[d] \setminus \{i\}}$. The ``$\Leftarrow$'' direction is straightforward (see appendix), so here we focus our attention on the ``$\Rightarrow$'' direction.
Suppose $\Phi_i f = 0$, which means $A_if = B_if$. Property~2 of Lemma~\ref{lemma:key_properties} implies $B_if \in F_{[d] \setminus \{i\}}$, and thus $A_if \in F_{[d] \setminus \{i\}}$. Thus, it suffices to show that the pre-image of $F_{[d] \setminus \{i\}}$ (denoted $A_i^{-1}\left(F_{[d] \setminus \{i\}} \right)$) is a subset of $F_{[d] \setminus \{i\}}$. 

Doing so is particularly simple when $F$ (and therefore $F_{[d] \setminus \{i\}}$) is finite dimensional. Property~3 of Lemma~\ref{lemma:key_properties} implies that $A_i$ has a trivial kernel, which means that $A_i$ has a well defined inverse $A_i^{-1}$. Property~1 of Lemma~\ref{lemma:key_properties} implies that $A_i\left(F_{[d] \setminus \{i\}}\right) \subseteq F_{[d] \setminus \{i\}}$. 
In the case where $F$ is finite dimensional, it then follows that $\dim \left(A_i\left(F_{[d] \setminus \{i\}}\right)\right) = \dim \left(F_{[d] \setminus \{i\}}\right)$, which implies that the two vector spaces must be equal. Thus $A_i$ is an injective and surjective map from $F_{[d] \setminus \{i\}}$ to itself, which means that it must be bijective, which implies $A_i^{-1}\left(F_{[d] \setminus \{i\}} \right) \subseteq F_{[d] \setminus \{i\}}.$
To handle the infinite-dimensional case, it turns out there is a technical trick one can use to reduce it to the finite dimensional case. We defer this to Appendix~\ref{app:main_theorem_proof}.
\end{proof}

\subsection{Discarding features based on close-to-zero aggregate SHAP values}\label{sec:robust_theorem}

Theorem~\ref{main_theorem} has two drawbacks. First, it requires SHAP values to be \textit{exactly} equal to $0$. Second, it considers all points in the entire extended support. Thus translating it into a statement about aggregate SHAP values is not immediately obvious, because aggregate SHAP values are averaged over the support of the original distribution.
We address both of these issues by replacing $\mu$ with the extended distribution $\mu^*$. That is, we propose that aggregate SHAP values be computed with respect to the extended distribution. This immediately addresses the second issue as these aggregate values \textit{will} take the full extended support into account. It turns out, this idea also addresses the first issue, allowing for a more flexible bound.


\begin{theorem}[Small  $\mu^*$-SHAP value allows to discard feature]\label{theorem:robust_distribution_bound}
Let $\mu$ be a distribution on $\R^d$, and $f: \R^d \to [0, 1]$ a measurable function. Let $1 \leq i \leq d$ be a feature. Suppose that the aggregate SHAP value, $\overline{\phi_i}(\mu^*, f) \leq \epsilon.$ Then there exists $g \in F_{[d] \setminus \{i\}}$ s.t. $\int \left(f(x) - g(x)\right)^2d\mu^*(x) < d^2\epsilon.$
\end{theorem}

Observe here that the SHAP values are both averaged over and computed with $\mu^*$. In addition to encompassing the entire extended support, we will see that $\mu^*$ also lends itself to a tighter analysis due to its features being independent. \\

\begin{proof}[Theorem~\ref{theorem:robust_distribution_bound} (Sketch); full proof in Appendix \ref{proof:theorem:robust_distribution_bound}]
Recall that $F$ denotes the space of all measurable functions $supp(\mu^*) \to \C$. The key observation is to define an inner product over $F$ with $\langle f_1, f_2 \rangle = \int \overline{f_1(x)}f_2(x) d\mu^*(x).$ We can then show that over $\mu^*$, the value operators $v_S$ are \textit{Hermitian}. From here, we can essentially follow the proof of Theorem \ref{main_theorem}. The only difference is that when we apply Lemma~\ref{lemma:key_properties}, we can additionally bound the eigenvalues of $A_i^{-1}$ (thus strengthening Property~3 of Lemma~\ref{lemma:key_properties}). We then conclude by arguing that if $(A_i - B_i)f$ is close to $0$, then $A_if$ is close to $F_{[d] \setminus \{i\}}$. This means that the distance from $f$ to $F_{[d] \setminus \{i\}}$ can be bounded with the norm of $A_i^{-1}$, which in turn is bounded based on its eigenvalues (as it too is Hermitian).
\end{proof}

\section{Aggregate SHAP Values in the Finite Sample Setting
}\label{section:kernel_shap}

Thus far, our results have been in the distributional setting where SHAP values and value functions are both computed based on the true expectations taken over $\mu$ (or $\mu^*$). Hence, as a next step we will study the \textit{finite sample regime}, where SHAP values are computed with respect to an i.i.d sample $X = \{x^{(1)}, \dots, x^{(n)}\} \sim \mu^n$. 
A natural way to approximate SHAP values in this setting is to replace true expectations with their corresponding empirical estimates. However, this is computationally infeasible in practice due to the exponential number of subsets $S$ ($2^d$ total) one must consider. The most popular method to address this issue is KernelSHAP~\citep{LundbergLee2017}, which uses weighted linear regression to approximate the value function $v_S(\mu, f, x)$, and then combines these approximations to obtain a tractable estimate of the SHAP value $\phi_i(\mu, f, x)$. For the purposes of proving our results, we include a detailed definition of KernelSHAP in Appendix \ref{app:kernel_shap_definition}. 

We denote the KernelSHAP value for Feature $i$ at point $x \in \R^d$ with $\kr_i(X, f, x)$, where $X = \{x^{(1)}, \dots, x^{(n)}\}$ is a set of $n$ points in $\R^d$. We also denote the aggregate KernelSHAP value as $\overline{\kr}_i(X, f)$ which is defined as $$\overline{\kr}_i\left( X = \{x^{(1)}, \dots x^{(n)}\}, f\right) = \frac{1}{n}\sum_{j= 1}^n |\kr_i(X, f, x^{(j)})|.$$ 

We now turn to our main objective, which is to find an analog of Theorem \ref{theorem:robust_distribution_bound} that applies to KernelSHAP. Recall that the main idea from the previous section was that SHAP values must be computed over the \textit{extended distribution} $\mu^*$ in order to achieve soundness. This idea will also apply to KernelSHAP in a similar way. To use the extended distribution, we need to replace the training sample $X \sim \mu^n$ with a sample from $(\mu^*)^n$. Although this cannot be directly done (as typically users only have access to samples from $\mu$), it turns out that simply scrambling the columns of the data matrix $X$ (as shown in Algorithm \ref{alg:SHAP_extended_support}) suffices. More precisely, we let $X^* = \{(x^*)^{(1)}, \dots, (x^*)^{(n)}\}$ be the dataset constructed as follows: if $\sigma_1, \dots, \sigma_d$ are independent random permutations of $[n]= \{1, \dots, n\}$, then 
\begin{equation}\label{eqn:scramble}
(x^*)^{(j)}_i = x^{(\sigma_i(j))}_i: 1 \leq i \leq d, 1 \leq j \leq n. 
\end{equation}
We now investigate the soundness of running KernelSHAP over this scrambled dataset. To do so, we will express our results in terms of an error term, denoted $\eta(X^*, \mu^*, f)$. This term represents how far the aggregate SHAP values are from the values that \textit{KernelSHAP} converges towards in the large sample limit. Crucially, this term has no relevance to the true SHAP values -- it is rather a reflection of the convergence behavior that KernelSHAP exhibits when it is applied over a large dataset. This quantity is extensively studied in \citep{CovertLee21}, and has been shown to be quite small both theoretically and practically. We include a full discussion of this in Appendix \ref{app:kernel_shap_definition}. 
\begin{theorem}[Small $\mu^*$-aggregate KernelSHAP Value allows to discard features]\label{thm:kernel_shap}
Let $\mu$ be a \linebreak distribution and $f: \R^d \to [0, 1]$ a measurable function. Let $1 \leq i \leq d$ be a feature. Let $X = \{x^{(1)}, \dots, x^{(n)}\} \sim \mu^n$ denote an i.i.d. sample of $n$ points from $\mu$, and let $X^*$ be as defined in Equation~\ref{eqn:scramble}. Suppose that $\overline{\kr}_i(X^*, f) \leq \epsilon$. Then there exists $g \in F_{[d] \setminus \{i\}}$ such that $$\int\left(f(x) - g(x)\right)^2 d\mu^*(x) < d^2\left(\epsilon + \eta(X^*, \mu^*, f)\right),$$
where $\eta(X^*, \mu^*, f)$ denotes the error between the empirical computation of KernelSHAP and its limit object (see Definition~\ref{defn:error_term}).
\end{theorem}

Theorem \ref{thm:kernel_shap} has \textit{direct} implications for practitioners: implementing a procedure for constructing $X^*$ is trivial and this is the \textit{only} modification needed for KernelSHAP to enjoy similar soundness guarantees as SHAP does. 
We now briefly sketch a proof, with full details deferred to Appendix~\ref{proof:thm:kernel_shap}.\\

\begin{proof}[Theorem~\ref{thm:kernel_shap} (Sketch); full proof in Appendix~\ref{proof:thm:kernel_shap}]
We define an operator, $\kr_i$, that corresponds to the limit object of KernelSHAP. The crux of this proof is to use the explicit formula for $\kr_i$ given in \cite{CovertLee21} to show that much like SHAP operator $\Phi_i$, $\kr_i$ can be expressed as a linear combination of value operators. This allows us to leverage an analog of Lemma~\ref{lemma:key_properties}. 
At a high level, the limit object of KernelSHAP is the solution to a particular linear regression that attempts to predict value functions. Using the standard formula for solving a linear regression, we see that this solution is \textit{linear} with respect to the target vector. This implies that $\kr_i$ itself is linear with respect to the value functions. 
Finally, we show that this linear combination satisfies an analog of Lemma~\ref{lemma:key_properties}. This follows from straightforward algebraic manipulations, beginning with an explicit formula for $\kr_if$ derived in \cite{CovertLee21}. 
\end{proof}

\section{Technical Toolbox: The Shapley Lie Algebra}\label{sec:lie_algebra}

We now develop the technical tools that allow us to prove Lemma~\ref{lemma:key_properties}, the key ingredient in our proofs.  
In Section~\ref{sec:value_properties}, we prove Properties 1 and 2 by studying the relationship between value operators and determined spaces.   Observe that the operators $A_i$ and $B_i$ (Definition~\ref{defn:shapley_operator}) are linear combinations of value operators, and thus their behavior over determined spaces can be characterized based on looking at value operators.  

In Section~\ref{sec:shapley_algebra}, we develop the technical machinery for proving Property~3. The core difficulty is to analyze the invertibility of a linear combination of linear operators. One natural idea for doing so would be to attempt to simultaneously diagonalize the operators. This would allow us to only consider diagonal matrices, which would greatly simplify the problem.
Unfortunately simultaneous diagonalization is only possible for a set of commuting matrices. 
To circumvent this, we will appeal to Lie Theory which provides tools to study families of transformations that \textit{almost commute} with each other. First, we construct a Lie algebra generated by the value operators and show that it is \textit{solvable} (which can be thought of as a generalization of commutative). Second, we apply Lie's theorem to find a basis in which all value operators are simultaneously \textit{upper triangular}. This enables us to show that $A_i$ is also upper triangular, which in turn demonstrates its invertibility.  

\subsection{Properties of Value Operators}\label{sec:value_properties}

Recall that the idea behind the value function $v_S(\mu, f, x)$ is to represent what $f$ would output at $x$ if it only had access to the coordinates from $x_S$. Applying this over all $x \in supp(\mu^*)$ suggests that applying the value operator $v_S$ to $f$ results in a function $v_Sf$ that is only impacted by the coordinates of its input from $S$. Phrasing this in terms of determined function spaces, this can be written as $v_Sf \in F_S$. We now study the more general problem of characterizing how $v_S$ behaves over an arbitrary determined function space $F_T$ for some other set $T \subseteq [d]$. 

\begin{lemma}[Images and Eigenspaces of Value Operators]\label{lem:value_operators_awesome}
If $S, T \subseteq [d]$, then the following hold:
\begin{enumerate}[noitemsep]
	\item \textbf{Image of $v_S$:} $v_S(F_T) \subseteq F_{S \cap T}$.
	\item \textbf{Eigenspace of $v_S$:} If $T \subseteq S$, then $v_Sf = f$ for all $f \in F_T$. 
\end{enumerate}
\end{lemma}

This lemma is a straightforward consequence of the definitions of value operators and determined function spaces. We defer a proof to Appendix \ref{app:this_is_the_last_proof_please}.

Lemma~\ref{lem:value_operators_awesome} shows that the operator $v_S$ essentially projects all determined function spaces into their image within $F_S$. Furthermore, it implies Properties 1 and 2 of Lemma~\ref{lemma:key_properties}. Property 1 holds since \textit{any} value operator must map $F_{[d] \setminus \{i\}}$ to itself meaning the same holds for any linear combination of value operators (such as $A_i$). Property 2 holds since the specific value operators that comprise $B_i$ all corresponds to subsets of $[d]\setminus \{i\}$, which means their images must all be constrained to $F_{[d] \setminus \{i\}}$. 

\subsection{The Shapley Lie Algebra}\label{sec:shapley_algebra}

We begin by defining the Shapley Lie algebra.

\begin{definition}[Shapley Lie Algebra]\label{defn:shapley_algebra}
The Shapley Lie algebra $\sg$ is the Lie algebra generated by $\{v_S: S \subseteq [d]\}$. That is, $\sg$ is the smallest set of linear operators containing all $v_S$ such that for all $v, w \in \sg$, 
\begin{enumerate}[noitemsep]
	\item $\forall a, b \in \C$, $a v + bw \in \sg$,
	\item $[v, w] = vw - wv$ is also in  $\sg$. 
\end{enumerate}
\end{definition}

The operation $[v, w]$ is called the \textit{derivation} of $v$ and $w$. Observe that $v, w$ commute if and only if $[v, w] = 0$. More broadly, $[v, w]$ can be thought of as representing the degree to which $v$ and $w$ commute. This idea is expressed through  \textit{solvability}, which can be thought of as a generalization of commutativity. We now state the main result of this section, that the Shapley Lie algebra is solvable. 

\begin{lemma}[Shapley Lie Algebra is Solvable]\label{lem:shap_lie_solvable}
The Shapley Lie algebra is solvable, meaning that the following holds:
\begin{enumerate}[noitemsep]
	\item For any Lie algebra, $\g$, its derivation $[\g, \g]$ is the Lie sub-algebra generated by $\{[v_1, v_2]: v_1, v_2 \in \g\}$.
	\item Define $\sg$'s derived series is the sequence $\sg^{(i)} = [\sg^{(i-1)}, \sg^{(i-1)}]$ with $\sg^{(0)} = \sg$.
	\item Then there exists $n$ such that $\sg^{(n)} = 0$. 
\end{enumerate}
\end{lemma}

The length of a Lie algebra's derived series serves as a measure of how ``far" the algebra is from being commutative. 

To prove Lemma~\ref{lem:shap_lie_solvable}, we begin by restricting our attention to elements of $\sg$ that can be constructed from value operators \textit{purely} through derivations, rather than through derivations and linear combinations. Such elements will prove easier to analyze using Lemma~\ref{lem:value_operators_awesome}. 

\begin{definition}[Pure operators]\label{defn:pure_operator}
The set of pure operators in $\sg$ is defined as the smallest subset $V \subseteq \sg$ that is closed under derivations and also contains all value operators. 
\end{definition}

The key idea to showing that $\sg$ is solvable is generalize Lemma~\ref{lem:value_operators_awesome} to apply to pure operators. 

\begin{lemma}[Images and Eigenspaces of Pure Operators]\label{lem:images_of_pure_operators}
Let $V \cap \sg^{(1)}$ denote all pure operators that are in the derived Lie sub-algebra of $\sg$. For all pure operators $v \in V \cap \sg^{(1)}$, there exists a subset $\alpha(v) \subseteq [d]$ such that the following hold.
\begin{enumerate}[noitemsep]
	\item \textbf{Image of $v$:} $v(F_T) \subseteq F_{\alpha(v) \cap T}$.
	\item \textbf{Eigenspace of $v$:} If $T \subseteq \alpha(v)$, then $vf = 0$ for all $f \in F_T$. 
	\item \textbf{Behavior with Derivations:} $\alpha([v, w]) \subseteq \alpha(v) \cap \alpha(w)$. 
\end{enumerate}
\end{lemma}

Observe that unlike Lemma~\ref{lem:value_operators_awesome}, the eigenspace here operates with an eigenvalue of $0$ rather than $1$. This results from the way derivations are a difference between two operators. As a sanity check, when $\sg$ is commutative, the lemma trivially holds with $\alpha(v) = \emptyset$ for all derived elements $v$.

The proof idea for Lemma~\ref{lem:images_of_pure_operators} is to partition the set of pure operators into levels based upon how many derivations are needed to construct a given element (the value operators are on the $0$th level). Then, beginning with Lemma~\ref{lem:value_operators_awesome}, apply induction on the level. We defer a proof to Appendix~\ref{proof:lem:images_of_pure_operators}. 

We now provide a proof sketch for Lemma~\ref{lem:shap_lie_solvable}.

\begin{proof}[Lemma \ref{lem:shap_lie_solvable} (Sketch); full proof in Appendix~\ref{proof:lem:shap_lie_solvable}]
The key idea is to characterize the way applying derivations to pure elements effects the associated subset $\alpha(v)$. By repeatedly applying Lemma~\ref{lem:images_of_pure_operators}, we can show $$\alpha([v, w]) \subseteq \alpha(v) \cap \alpha(w), \text{ and }\alpha(v) = \alpha(w) \implies [v, w] = 0.$$

Observe that together, these two statements imply that $\alpha([v, w])$ is strictly smaller in cardinality than $\max\left(|\alpha(v)|, |\alpha(w)|\right)$. From here it is relatively straightforward to see that applying $(d+1)$ successive derivations to any set of pure operators will result in $0$. To finish the result, we simply show (through checking definitions) that pure operators form bases (in the linear algebra sense) of all Lie sub-algebras in the sequence $\sg^{(0)}, \sg^{(1)}, \dots.$ Thus, since the pure operators contained in $\sg^{(d+1)}$ are $0$, it follows that $\sg^{(d+1)} = 0$ which implies solvability.
\end{proof}

We conclude this section by sketching a proof of Property~3 of Lemma~\ref{lemma:key_properties}. 

\begin{proof}[Property 3 of Lemma \ref{lemma:key_properties} (Sketch); full proof in Appendix \ref{app:lemma:key_properties_proof}] 
As before, assume that $F$, the space over which all our operators act, is finite dimensional. Extending our argument to the infinite dimensional case is handled in the appendix. It then follows by an application of Lie's theorem (Theorem \ref{theorem:LIE_THEOREM}) in conjunction with Lemma \ref{lem:shap_lie_solvable} that there exists a basis of $F$ over which all operators in $\sg$ are \textit{simultaneously upper triangular}. Thus, $A_i = \frac{1}{d}\sum_{S \subseteq [d] \setminus \{i\}} \binom{d-1}{|S|}^{-1} \upsilon_{S \cup \{i\}}$ can be written as a strictly positive sum of upper triangular matrices, and is itself upper triangular. Finally, since $v_{[d]}$ is included in this sum, and since $v_{[d]}$ \textit{is} the identity operator, it follows that $A_i$ has a strictly positive diagonal which implies that it is invertible.
\end{proof}

\section{Discussion}

The main goal of this work is to investigate the soundness of the widely used practice of aggregating SHAP values. We show that provided they are computed over the extended distribution, SHAP and KernelSHAP values can be used to soundly eliminate unimportant features. We stress that our results are not intended to suggest these algorithms as a first choice  for eliminating features -- there exist other better methods for doing so. We instead contend that our results guarantee soundness in settings where SHAP and KernelSHAP are being routinely applied --- provided practitioners adopt our modification to aggregate over the extended data support. In practice, this modification is straightforward to implement and does not require to change the SHAP packages' internal code. One  only has to replace a sample from $\mu$ with a sample from $\mu^*$, which can easily be achieved by scrambling data columns appropriately (see Algorithm \ref{alg:SHAP_extended_support}). 

Our techniques may also be useful for analyzing other properties of SHAP as well. As a testament to their versatility, they readily apply to both interventional SHAP values and KernelSHAP values, despite the latter lacking a formal connection to the former.

\acks{We thank Gunnar König for his helpful discussions on KernelSHAP and Sebastian Bordt for his feedback. 
This work has been supported by the German Research Foundation through the Cluster of Excellence “Machine Learning - New Perspectives for Science" (EXC 2064/1 number 390727645) and the Carl Zeiss Foundation through the CZS Center for AI and Law.
}

\bibliography{main.bib}
\newpage
\appendix


%
%
%

\section{Proofs from Section \ref{sec:aggregate_characterization}}

\subsection{Proof of Theorem \ref{main_theorem}}\label{app:main_theorem_proof}

Recall that in our proof sketch, we treated $F$ as though it were finite dimensional. To handle the infinite case, we will apply Lemma~\ref{lemma:localized_subspace}, that allows us to restrict our attention to a finite dimensional subspace $F_f \subseteq F$. A statement and proof of Lemma~\ref{lemma:localized_subspace} can be found in Section~\ref{app:finite_dimensional_stuff} of the appendix. We are now prepared to prove Theorem~\ref{main_theorem}.\\

\begin{proof}[Theorem \ref{main_theorem}]
It suffices to show that for all $f \in F$ and $1 \leq i \leq d$, $\Phi_if = 0$ if and only if $f \in F_{[d] \setminus \{i\}}$. 

($\Rightarrow$) Suppose $\Phi_i f = 0$ which means that $A_if = B_if$. Property 2 of Lemma~\ref{lemma:key_properties} implies ${B_if \in F_{[d] \setminus \{i\}}}$, and thus $A_if \in F_{[d] \setminus \{i\}}$. Our main idea will be to use a dimension counting argument to show that because $A_i$ is invertible (Property 3 of Lemma~\ref{lemma:key_properties}) and because it preserves $F_{[d] \setminus \{i\}}$ (Property 1 of Lemma~\ref{lemma:key_properties}), $f$ must itself be inside $F_{[d] \setminus \{i\}}$. This would be immediate from the Rank-Nullity theorem if $F_{[d] \setminus \{i\}}$ was finite dimensional. Unfortunately this is not quite the case as spaces of functions are typically infinite dimensional.

To circumvent this issue, we use Lemma \ref{lemma:localized_subspace} which states that $F_f$ is both $A_i$ and $B_i$-invariant. Let $F_f$ be as defined in Definition~\ref{app:defn:localized_subspace} and let $W = F_f \cap F_{[d] \setminus \{i\}}$. Let $W'$ be the vector space generated by $W$ and $f$. Then by restricting $A_i$ to $W'$ and by using the fact that $F_f$ is $A_i$-invariant along with property 1 of Lemma~\ref{lemma:key_properties}, we have $A_i(W') \subseteq W$. This now precisely corresponds to the finite dimensional case, and the Rank-Nullity theorem again implies $W' = W$ which means $f \in W \subseteq F_{[d] \setminus \{i\}}$ as desired. 

($\Leftarrow$) Suppose $f \in F_{[d] \setminus \{i\}}$. By directly utilizing the definition of SHAP values (Definition~\ref{defn:shap_value}), for all ${x \in \supp(\mu^*)}$ we have
\begin{equation*}
\begin{split}
\phi_i(\mu, f, x) &= \frac{1}{d}\sum_{S \subseteq [d] \setminus \{i\}} \binom{d-1}{|S|}^{-1} \left(v_{S \cup \{i\}}(f, x) - v_S(f, x)\right). \\
&= \frac{1}{d}\sum_{S \subseteq [d] \setminus \{i\}} \binom{d-1}{|S|}^{-1} \left(\Ev_{X \sim \mu}[f\left(x_{S \cup \{i\}}, X_{S^c \setminus \{i\}}\right)] - \Ev_{X \sim \mu}[f\left(x_S, X_{S^c}\right)]\right) \\
&= \frac{1}{d}\sum_{S \subseteq [d] \setminus \{i\}} \binom{d-1}{|S|}^{-1} \Ev_{X \sim \mu}\left[f\left(x_{S}, x_{\{i\}}, X_{S^c \setminus \{i\}}\right) - f\left(x_S, X_{\{i\}}, X_{S^c \setminus \{i\} }\right)\right].
\end{split}
\end{equation*}
Observe that the difference above is taken over $f$ evaluated at two points that only differ in their $i$th coordinate. Since $f$ is $[d] \setminus \{i\}$-determined, this must equal $0$ which implies the result. 
\end{proof}

\subsection{Proof of Lemma \ref{defn:value_operator}}\label{proof:defn:value_operator}

\begin{proof}[Lemma \ref{defn:value_operator}]
We first show it is well defined. The only cause for concern is that $v_S(\mu, f, x)$ was defined over functions $f: \R^d \to \R$ rather than $f: \supp(\mu^*) \to \R$. To resolve this, we expand out $v_S(\mu, f, x)$ using the definition of a value function (Definition~\ref{defn:value_function}). We have,
\begin{equation*}
(v_Sf)(x) = v_S(\mu, f, x) = \Ev_{X \sim \mu} [f \left(x_S, X_{S^c}\right)].
\end{equation*}
For any $i$, $x_i$ and $X_i$ lie within the support of $\mu_i^*$ (Definition \ref{defn:extended_distribution}) and thus $(x_S, X_{S^c})_i$ does as well. Since this holds for all $i$, it follows $(x_S, X_{S^c})_i \in supp(\mu_i^*)$ which implies $f \left(x_S, X_{S^c}\right)$ is well defined. 

Finally, the fact that $v_S$ is linear is an immediate consequence of the linearity of the expectation which concludes the proof. 
\end{proof}

\subsection{Proof of Lemma \ref{lemma:key_properties}}\label{app:lemma:key_properties_proof}

Recall that our strategy is to utilize Lemma \ref{lem:value_operators_awesome} to prove Properties 1 and 2, and Lemma \ref{lem:shap_lie_solvable} along with Lie's theorem (Theorem~\ref{theorem:LIE_THEOREM}) to prove Property~3. To do so, we begin by first proving a key technical lemma that will enable us to not only derive Property~3, but also assist in proving Theorem~\ref{theorem:robust_distribution_bound} (given in Section~\ref{sec:robust_theorem}). 

\begin{appxlemma}[Eigenvalues of $A_i$]\label{appxlemma:eigenvalues_of_A_i}
Let $f \in F$ be a non-zero function such that $A_if = \lambda f$ for some $\lambda \in \C$. Then $\lambda$ is a positive real number with $\lambda \geq \frac{1}{d}$. 
\end{appxlemma}

\begin{proof}[Lemma \ref{appxlemma:eigenvalues_of_A_i}]
Fix any $f \in F$ with $A_if = \lambda f$ for $\lambda \in \C$. By Lemma \ref{lem:definition_local_representation}, there exists a finite dimensional representation (see Definition \ref{defn:representation}) $(\rho_f, F_f)$ of $\sg$ such that
\begin{enumerate}
	\item $f \in F_f$,
	\item for all $v \in \sg$, $\rho_f(v): F_f \to F_f$ is the restriction of $v$ to $F_f$. In particular, this means $F_f$ is a $v$-invariant subspace for all $v \in \sg$.
\end{enumerate}

By Lemma \ref{lem:shap_lie_solvable}, $\sg$ is a solvable Lie algebra. Thus, we can apply Lie's theorem (Theorem \ref{theorem:LIE_THEOREM}), which implies that there exists a basis of $F_f$ over which all represented value operators $\rho_f(v_S)$ can be expressed as upper triangular matrices $M_S$. Furthermore, Lemma \ref{lem:value_operators_awesome} implies that for all $S \subseteq [d]$, $v_Sv_S = v_S$. Thus, all of the eigenvalues of $v_S$ are either $0$ or $1$, which implies diagonal elements of $M_S$ are either $0$ or $1$ as well. 

Since $A_i \in \sg$, we see that $A_i$ itself can be represented as the matrix $$M(A_i) = \frac{1}{d}\sum_{S \subseteq [d] \setminus \{i\}} \binom{d-1}{|S|}^{-1} M_{S \cup \{i\}}.$$ This matrix too is upper triangular, and its diagonal elements are all positive linear combinations of elements that are either $0$ or $1$. Thus, all diagonal elements of $M(A_i)$ are nonnegative real numbers. Furthermore, since $M_{[d]}$ itself is included in this sum, and since $v_{[d]}$ is simply the identity operator, it follows that \textit{every diagonal element} of $M(A_i)$ is at least $\frac{1}{d}\binom{d-1}{d-1}^{-1} = \frac{1}{d}$.

Finally, since $M(A_i)$ is upper triangular, its diagonal elements are precisely its eigenvalues. However, $f$ is included in $F_f$ and satisfies $A_if = \lambda f$. This implies $M(A_i)f = \lambda f$. Since $f$ is non-zero, $\lambda$ is an eigenvalue of $M(A_i)$. Thus $\lambda$ is a diagonal element of $M(A_i)$ and is a real number that is at least $\frac{1}{d}$, as desired. 
\end{proof}

We are now prepared to prove Lemma \ref{lemma:key_properties}.

\begin{proof}[Lemma~\ref{lemma:key_properties}]
Observe that $A_i$ and $B_i$ are both linear combinations of value functions, and thus elements of $\sg$. Thus, applying Property~1 of Lemma~\ref{lem:value_operators_awesome} implies that $A_i(F_S) \subseteq F_S$ for all $S$ (Property 1). Meanwhile, for all $S \subseteq [d] \setminus \{i\}$, the lemma similarly implies $$v_S(F) \subseteq F_{S \cap [d]} = F_S \subseteq F_{[d] \setminus \{i\}},$$ with the last inclusion holding since $S$-determined functions are clearly $[d]\setminus\{i\}$-determined as $S \subseteq [d]\setminus\{i\}$. Summing those over the definition of $B_i$ implies $B_i(F) \subseteq F_{[d] \setminus \{i\}}$ (Property 2). 

Finally, we prove Property 3. Fix $f \in F$ with $A_if = 0$. Then Lemma \ref{appxlemma:eigenvalues_of_A_i} implies that $f$ must be $0$ as otherwise $0$ would be an eigenvalue of $A_i$ that is smaller than $\frac{1}{d}$. This implies Property 3. 
\end{proof}

\subsection{Proof of Theorem \ref{theorem:robust_distribution_bound}}\label{proof:theorem:robust_distribution_bound}

We begin by precisely defining the operators that correspond to computing SHAP values over the \textit{entire} extended distribution (Definition~\ref{defn:extended_distribution}).

\begin{appdfn}[Value and SHAP Operators]\label{appdfn:operators}
Let $\mu$ be a distribution over $\R^d$ and $\mu^*$ its extended distribution. Let $F$ be the space of all measurable functions $supp(\mu^*) \to \C$. For $S \subseteq [d]$, we define the value operator $v_S^*$ as $$v_S^*f(x) = \Ev_{X \sim \mu^*}\left[f\left(x_S, X_{S^c}\right)\right].$$ We then define the SHAP operator $\Phi_i^*$ for $1 \leq i \leq d$ with  $$\Phi_i^*f = A_i^*f - B_i^*f = \frac{1}{d}\sum_{S \subseteq [d] \setminus \{i\}} \binom{d-1}{|S|}^{-1} v_{S \cup \{i\}}^*f - \binom{d-1}{|S|}^{-1} v_{S}^*f.$$
\end{appdfn}

The definitions of $v_S^*$, $A_i^*$, $B_i^*$ and $\Phi_i^*$ all directly correspond to Definitions \ref{defn:value_operator} and \ref{defn:shapley_operator}. More generally, we will use the $*$ notation to denote the analog of quantity when $\mu$ is replaced by $\mu^*$. The only thing to note is that the spaces of determined functions (Definition \ref{defn:det_func_vec_space}) we operate on, $F_S: S \subseteq [d]$, \textit{remain unchanged.} This is because the extended support of $\mu^*$ is simply itself, as $(\mu^*)^* = \mu^*$. This is clear from the definition of $\mu^*$. 

To avoid confusion, \textbf{we will use the $*$ notation in all cases \textit{except} determined function spaces, where we will omit the $*$.}

Next, we turn our attention towards the main business of this section which is proving Theorem~\ref{theorem:robust_distribution_bound}. To this end, we define an inner product over $F$ based on $\mu^*$. 

\begin{appdfn}[Inner product with $\mu^*$]\label{appdfn:inner_product}
For $f, g \in F$, we define $$\langle f, g \rangle = \int_{\supp(\mu^*)} \overline{f(x)}g(x) d\mu^*(x),$$ where $\mu^*$ denotes the extended distribution of $\mu$. 
\end{appdfn}

Here, $\overline{z}$ denotes the complex conjugate of $z \in \C$. It can be easily verified that this is a well-defined inner product that makes $(F, \langle -,-\rangle)$ a Hilbert space. Next, for $S \subseteq [d]$, recall that $v_S^*$ denotes the value operator taken with respect to $\mu^*$. That is, $$v_S^*f(x) = \Ev_{X \sim \mu^*}\left[f\left(x_S, X_{S^c} \right) \right].$$ The key idea for eventually proving Theorem~\ref{theorem:robust_distribution_bound} is to show that $v_S^*$ is Hermitian:

\begin{appxlemma}[Value operators over $\mu^*$ are Hermitian]\label{appxlem:hermitian}
For all $S \subseteq [d]$, $v_S^*$ is Hermitian with respect to the inner product given in Definition \ref{appdfn:inner_product}. That is, for all $f, g \in F$, $$\langle v_S^*f, g \rangle = \langle f,  v_S^* g \rangle.$$ 
\end{appxlemma}

\begin{proof}[Lemma \ref{appxlem:hermitian}]
Our main idea is to exploit the fact that $\mu^*$ is a product of $d$ independent distributions, $\mu_1^*, \dots, \mu_d^*$. For any $T \subseteq [d]$, let $\mu_T^* = \prod_{i \in T} \mu_i^*$. It follows that $\mu^* = \mu_T^* \times \mu_{T^c}^*$. To help simplify notation, for $p \in \supp(\mu_T^*)$ and $q \in \supp(\mu_{T^c}^*)$, we let $f(p, q)$ denote $f(x)$ where $x_T = p$ and $x_{T^c} = q$. 

Applying this along with the definition of value operators, we see that
\begin{equation*}
\begin{split}
\langle v_S^* f, g \rangle 
&= \int_{supp(\mu^*)} \overline{(v_S^*f)(x)}g(x) d\mu^*(x) \\
&= \int_{\supp(\mu^*)} \left(\int_{supp(\mu^*)} \overline{f\left(x_S, X_{S^c}\right)}d\mu^*(X) \right)g(x) d\mu^*(x) \\
&= \int_{\supp(\mu_S^*) \times \supp(\mu_{S^c}^*)}\left(\int_{\supp(\mu_{S^c}^*)}\overline{f(p, q)}d\mu_{S^c}^*(q)  \right) g(p, r)d\mu_S^*(p)d\mu_{S^c}^*(r) \\
&= \int_{supp(\mu_S^*) \times \supp(\mu_{S^c}^*) \times \supp(\mu_{S^c}^*)} \overline{f(p, q)}g(p, r) d\mu_S^*(p) d\mu_{S^c}^*(q) d\mu_{S^c}^*(r) \\
&= \int_{supp(\mu_S^*) \times \supp(\mu_{S^c}^*)} \left(\int_{\supp(\mu_{S^c}^*)} g(p, r)d\mu_{S^c}^*(r)\right)\overline{f(p, q)} d\mu_S^*(p) d\mu_{S^c}^*(q) \\
&= \int_{\supp(\mu^*)} \left(\int_{supp(\mu^*)} g\left(x_S, X_{S^c}\right)d\mu^*(X) \right)\overline{f(x)} d\mu^*(x) \\
&= \int_{supp(\mu^*)} (v_S^*g)(x) \overline{f(x)} d\mu^*(x) 
= \langle f, v_S^*g \rangle
\end{split}
\end{equation*}

Basically, expanding out the inner product gives an integral over 2 sets of variables, one drawn from $\mu^*$ corresponding to the expectation, and another drawn from $\mu_{S^c}^*$ corresponding to the value operator. Due to the independent nature of these variables, they can be freely reordered resulting in the manipulation above. 
\end{proof}

Next, we show how to relate $\overline{\phi_i}(\mu^*, f)$, which is related to the absolute value of $f$, to the norm, $\langle \Phi_i^*f, \Phi_i^*f \rangle$.

\begin{appxlemma}[Bounding Norm with Aggregate SHAP]\label{appxlemma:l_1_to_l_2}
Let $f \in F$ be a function such that $f(x) \in [0, 1]$ for all $x \in supp(\mu^*)$. Then for all $1 \leq i \leq d$, $\langle \Phi_i^*f, \Phi_i^*f \rangle \leq \overline{\phi_i}(\mu^*, f).$
\end{appxlemma}

\begin{proof}[Lemma \ref{appxlemma:l_1_to_l_2}]
Recall that $\Phi_i^*$ is an operator that maps $F$ to itself. We begin by bounding the range of $\Phi_i^*f$. To do so, for any $S \subseteq [d]$, observe that for all $x \in supp(\mu^*)$,
\begin{equation*}
(v_S^*f)(x) = \Ev_{X \sim \mu^*} \left[f \left(x_S, X_{S^c} \right)\right] \in [0, 1],
\end{equation*}
Since everything within the expectation is an application of $f$ which has range in $[0, 1]$. It immediately follows that $A_i^*f, B_i^*f \geq 0$, as $A_i^*, B_i^*$ are both positive linear combinations of value operators (Definition \ref{appdfn:operators}). To get upper bounds on the range of these functions, we see that
\begin{equation*}
\begin{split}
(A_i^*f)(x) &= \frac{1}{d}\sum_{S \subseteq [d] \setminus \{i\}} \binom{d-1}{|S|}^{-1} (v_{S \cup \{i\}}^*f)(x) \\
&\leq \frac{1}{d}\sum_{S \subseteq [d] \setminus \{i\}} \binom{d-1}{|S|}^{-1} \\
&= \frac{1}{d} \sum_{j=0}^{d-1} \binom{d-1}{j}^{-1} \sum_{S \subseteq [d] \setminus \{i\}, |S| = j} 1 \\
&= \frac{1}{d} \sum_{j=0}^{d-1} \binom{d-1}{j}^{-1} \binom{d-1}{j} \\
&= 1. 
\end{split}
\end{equation*}
An analogous argument shows $(B_i^*f)(x) \leq 1$. It follows that $(\Phi_i^*f)(x) = (A_i^*f)(x) - (B_i^*f)(x)$ must be an element in $[-1, 1]$. Substituting this, we find that 
\begin{equation*}
\begin{split}
\overline{\phi_i}(\mu^*, f) &= \int_{supp(\mu^*)} |(\Phi_i^*f)(x)|d\mu^*(x) \\
&\geq \int_{supp(\mu^*)} |(\Phi_i^*f)(x)|^2 d\mu^*(x) \\
&= \langle \Phi_i^*f, \Phi_i^*f \rangle.
\end{split}
\end{equation*}
\end{proof}

We are now prepared to prove Theorem \ref{theorem:robust_distribution_bound}.\\

\begin{proof}[Theorem \ref{theorem:robust_distribution_bound}]

let $F_f^*$ be the finite dimensional subspace defined in Lemma \ref{lemma:localized_subspace} that corresponds to $\mu^*$ (the subspace in the lemma was defined for an arbitrary measure $\mu$). Let $W = F_f^* \cap F_{[d] \setminus \{i\}}$. Then Lemmas \ref{lemma:localized_subspace} and  \ref{lemma:key_properties} imply that
\begin{enumerate}
	\item $A_i^*(F_f^*) \subseteq F_f^*$,
	\item $A_i^*(W) \subseteq W$,
	\item $B_i^*(F_f^*) \subseteq W$,
	\item $(A_i^*)^{-1}\left(\{0\}\right) = \{0\}$. 
\end{enumerate}
Let $a_i^*$ and $b_i^*$ denote the restrictions of $A_i^*$ and $B_i^*$ to $F_f^*$. It follows that these too are well defined operators that map $F_f^* \to F_f^*$, and also satisfy $$(a_i^* - b_i^*)(h) = (A_i^* - B_i^*)(h) = \Phi_i^*h,$$ for all $h \in F_f^*$. Since $F_f^*$ is finite dimensional, it follows that $a_i^*$ has an inverse, $(a_i^*)^{-1}$. 

Next, Lemma \ref{appxlem:hermitian} implies that every value operators, $v_S^*$ is Hermitian, which implies that $A_i^*$ must be as well (as it is a linear combination of value operators).  Since $W \subseteq F_f^* \subseteq F$, they inherit the inner product structure from $F$, and it follows that $a_i^*$ and $(a_i^*)^{-1}$ are Hermitian as well. Since $A_i^*$ has real eigenvalues that are all at least $\frac{1}{d}$ (Lemma \ref{appxlemma:eigenvalues_of_A_i}), it follows that $(a_i^*)^{-1}$ has maximum eigenvalue at most $d$. It follows by standard linear algebra that for all $h \in F_f^*$, 
\begin{equation}\label{eqn:eigenvector_bounded_thing_hermitian}
\langle (a_i^*)^{-1}h, (a_i^*)^{-1}h \rangle \leq d^2 \langle h, h \rangle.
\end{equation}
We are finally ready to prove Theorem \ref{theorem:robust_distribution_bound}. We claim $g = (a_i^*)^{-1} b_i^*f$ suffices. Observe that this is well defined as $f \in F_f^*$ (Lemma \ref{lemma:localized_subspace}) and $a_i^*, (a_i^*)^{-1},$ and $b_i^*$ are all well defined over this space.

To show that $g$ suffices, we must show that $g \in F_{[d] \setminus \{i\}}$ and that $\langle f-g, f-g \rangle \leq d\epsilon$. For the first claim, we apply the 4 properties that we derived at the beginning of this proof. First, $b_i^*f = B_i^*f \in W$ by Property~3. Second, Properties $2$ and $4$ imply that $W$ is an $a_i^*$ invariant subspace. Since $W$ is finite dimensional and since $a_i^*$ is invertible, it follows that $W$ is also $(a_i^*)^{-1}$ invariant. Thus $(a_i^*)^{-1}b_i^*f \in W$. This implies $g \in W \subseteq F_{[d] \setminus \{i\}}$, as desired. 

For the second claim, we use the fact that $\overline{\phi_i}(\mu^*, f) \leq \epsilon$. By Lemma~\ref{appxlemma:l_1_to_l_2}, this implies that $\langle \Phi_i^* f, \Phi_i^* f \rangle \leq \epsilon$. Using this, we see that
\begin{equation*}
\begin{split}
\int_{\R^d} \left(f(x) - g(x)\right)^2 d\mu^*(x) &= \langle f - g, f- g \rangle \\
&= \langle f - (a_i^*)^{-1}b_i^*f, f - (a_i^*)^{-1}b_i^*f \rangle \\
&= \langle (a_i^*)^{-1} (a_i^*f - b_i^*f), (a_i^*)^{-1} (a_i^*f - b_i^*f) \rangle\\
&\leq d^2 \langle a_i^*f - b_i^*f, a_i^*f - b_i^*f \rangle \\
&= d^2\langle \Phi_i^*f, \Phi_i^*f \rangle \leq d^2 \epsilon.
\end{split}
\end{equation*}
Here we simply substitute Equation \ref{eqn:eigenvector_bounded_thing_hermitian} to simplify the inner product. This completes the proof. 
\end{proof}

\section{Proofs from Section \ref{section:kernel_shap}}

\subsection{Preliminaries on KernelSHAP}\label{app:kernel_shap_definition}

The main idea of KernelSHAP~\citep{LundbergLee2017} is to use linear regression to \textit{approximate} the value function. To construct KernelSHAP values at a point $x$, KernelSHAP learns a weight vector, $\kr^x \in \R^d$, such that for all $S \subseteq [d]$, 
\begin{equation}\label{eqn:key_linear_regression}
v_S(\mu, f, x) \approx \langle \kr^x, \one_S \rangle + \Ev_{X \sim \mu} [f(X)],
\end{equation} 
where $\one_S$ is an indicator vector for $S$ with $(\one_S)_i = \ind(i \in S)$. The key observation is that if this approximation was a precise equality, then 
\begin{equation*}
\begin{split}
\phi_i(\mu, f, x) &= \frac{1}{d}\sum_{S \subseteq [d] \setminus \{i\}}\binom{d-1}{|S|}^{-1}\left( v_{S \cup \{i\}}(\mu, f, x) - v_{S}(\mu, f, x)\right) \\
&= \frac{1}{d}\sum_{S \subseteq [d] \setminus \{i\}}\binom{d-1}{|S|}^{-1}\left( \langle \kr^x, \one_{S \cup \{i\}} \rangle + \Ev_{X \sim \mu}[f(x)] -\langle \kr^x, \one_{S} \rangle - \Ev_{X \sim \mu}[f(x)]\right) \\
&= \frac{1}{d}\sum_{S \subseteq [d] \setminus \{i\}}\binom{d-1}{|S|}^{-1} \langle \kr^x, \one_{S \cup \{i\}} - \one_S \rangle \\
&= \frac{1}{d}\sum_{S \subseteq [d] \setminus \{i\}}\binom{d-1}{|S|}^{-1} \kr_i^x = \kr_i^x.
\end{split}
\end{equation*}
For this reason, KernelSHAP simply outputs the coefficients of the linear regression as its approximations to the SHAP values. 

KernelSHAP solves the linear regression suggested in Equation \ref{eqn:key_linear_regression} by using OLS as follows. Let $X = \{x^{(1)}, \dots, x^{(n)}\} \sim \mu^n$ be an i.i.d sample from $\mu$, $Z = \left\{S_1, \dots, S_n\right\} \sim \pi^n$ be an i.i.d sample of points from a probability distribution $\pi$ over $[d]$ defined by 
\begin{equation}\label{eqn:pi_is_defined}
\pi\left(S\right) \propto \begin{cases} \frac{d-1}{\binom{d}{|S|}|S|\left(d -|S|\right)} & 0 < |S| < d \\ 0 & \text{otherwise}\end{cases}.
\end{equation}
Here, $Z$ is a set of i.i.d subsets from $[d]$ that are weighted in likelihood according to how prevalent they are in the weighting scheme used in the definition of SHAP. Note that the weights for sets of size $d$ and $0$ is $0$ -- this is because the values of $\langle \kr^x, \one_{[d]} \rangle$ and $\langle \kr^x, \one_\emptyset \rangle$ are enforced directly using constraints as their corresponding value functions can be more precisely estimated. It is for this reason that the weights given by $\pi_S$ do not \textit{exactly} match the weights that appear in Definition \ref{defn:shap_value}. For a more extended discussion of this choice, see \cite{LundbergLee2017}.

Putting it all together, KernelSHAP solves the following OLS regression. For $x \in \R^d$, $f: \R^d \to \R$, and $X, Z$ sampled as above, 
\begin{equation}\label{eqn:kernel_shap_optimization}
\begin{split}
\left(\kr_1(X, f, x), \dots, \kr_d(X, f, x)\right) = \argmin_{\kr^x \in \R^d} \frac{1}{n} &\sum_{j=1}^n \left( f\left(x_{S_j}, x_{S_j^c}^{(j)}\right) - \left\langle \kr^x, \one_{S_j} \right\rangle - \overline{f}\right)^2 \\
\text{such that} \quad &\overline{f} = \frac{1}{n}\sum_{j=1}^n f\left(x^{(j)}\right) \\
\quad &\langle \one_{[d]}, \kr^x \rangle = f(x) - \overline{f}.
\end{split}
\end{equation} 

In their analysis of KernelSHAP, \cite{CovertLee21} provide explicits for $\kr_i(X, f, x)$ and $\kr_i(\mu, f, x)$, where the latter is a limit object that the former converges towards in the large sample limit (see Definition~\ref{defn:kern_limit_object}). Both of these expressions will be extremely useful in our analysis, and we include them here.

\begin{lemma}[Solution to KernelSHAP: Equation 7 of \cite{CovertLee21}]\label{defn:kern_implementation_object}
Let $x \in \R^d$ be a point, $f: \R^d \to \R$ be a function, $X = \{x^{(1)}, \dots, x^{(n)}\}$ a dataset, and $Z = \{S_1, \dots, S_n\} \sim \pi^n$ be a set of subsets drawn according to $\pi$ (Equation \ref{eqn:pi_is_defined}). Then the solution to Equation \ref{eqn:kernel_shap_optimization} can be written as follows. Let $M_n \in \R^{d \times d}$ and $b_n \in \R^d$ be defined as $$M_n = \frac{1}{n}\sum_{j=1}^n \one_{S_j}\one_{S_j}^t \text{ and }b_n = \frac{1}{n}\sum_{j=1}^n \one_{S_j}\left(f\left(x_{S_j}, x_{S_j^c}^{(j)}\right) - \overline{f}\right).$$ Then the vector $\kr(X, f, x) = \left(\kr_1(X, f, x), \dots, \kr_d(X, f, x)\right)$ is equal to $$\kr(X, f, x) = M_n^{-1}\left(b_n - \one \frac{\one^tM_n^{-1}b_n - f(x) + \overline{f}}{\one^tM_n^{-1}\one} \right),$$ where $\one$ is shorthand for the all ones vector, $\one_{[d]} \in \R^d$. 
\end{lemma}

\begin{definition}[Limit Object of KernelSHAP: Equation 8 of \citet{CovertLee21}]\label{defn:kern_limit_object}
Let $x \in \R^d$ be a point. Let $\mu$ be a distribution over $\R^d$ and $f: \R^d \to \R$ a function. Let $M \in \R^{d \times d}$ and $b \in \R^d$ be defined as $$M = \Ev_{S \sim \pi} \left[\one_S\one_S^t\right] \text{ and }b = \Ev_{S \sim \pi} \left[\one_{S}\left(v_S(\mu, f, x) - v_{\emptyset}(\mu, f, x)\right)\right].$$ For any $x \in \R^d$, the limit object of KernelSHAP, $\kr(\mu, f, x) = \left(\kr_1(\mu, f, x), \dots, \kr_d(\mu, f, x)\right)$, is defined as: $$\kr(\mu, f, x) = M^{-1}\left(b - \one \frac{\one^tM^{-1}b - f(x) + v_\emptyset(\mu, f, x)}{\one^tM^{-1}\one} \right).$$
\end{definition}

We now define the error term that represents how quickly KernelSHAP converges. For our purposes, we are interested in the behavior of KernelSHAP when it is run on $X^*$, which is the data matrix obtained by scrambling the columns of $X$ (see Section \ref{section:kernel_shap}).

\begin{definition}[Error Term]\label{defn:error_term}
We define the error term between the empirical computation of KernelSHAP and its limit object as follows: $$\eta(X^*, \mu^*, f) = \max_{1 \leq i \leq d}\left|\left(\frac{1}{n}\sum_{x^{(j)} \in X^*} |\kr_i(X^*, f, x^{(j)})|\right) - \Ev_{x \sim \mu^*} |\kr_i(\mu, f, x)|\right|.$$
\end{definition}

While it is clear from the law of large numbers that $\eta(X^*, \mu^*, f) \to 0$ as $n \to \infty$, the precise rate of convergence isn't clear. This problem is extensively studied in \cite{CovertLee21}, where they give some rates of convergence along with plenty of empirical evidence that this rate is very fast in practice. For this reason, we express our result in terms of $\eta(X^*, \mu^*, f)$. 

Finally, we conclude this section by citing the following useful explicit formula for $M$ (given by \cite{CovertLee21}). 

\begin{lemma}[Explicit formula for $M$: \cite{CovertLee21}]
Let $I_d$ denote the $d \times d$ identity matrix, and $J_d$ denote the $d \times d$ matrix consisting of all ones. Then $M = pI_d + qJ_d$ where $$p = \frac{1}{2} - q\text{ and }q = \frac{1}{d(d-1)}\frac{\sum_{k=2}^{d-1} \frac{d-1}{d-k}}{\sum_{k=1}^{d-1} \frac{1}{k(d-k)}}.$$
\end{lemma}

\subsection{Proof of Theorem \ref{thm:kernel_shap}}\label{proof:thm:kernel_shap}

Let $\mu$ be fixed. We begin by defining an operator that corresponds to the limit object of KernelSHAP when taken over $\mu^*$. As before, we let $F$ denote the space of all functions $\supp(\mu^*) \to \mathbb{C}$. 

To simplify our algebra, we will also simply assume that $v_\emptyset^*f = 0$ -- this can be accomplished by simply subtracting the mean of $f$ from it. Doing so does not effect any of the SHAP values (and is in fact a common preprocessing step). 

\begin{definition}\label{def:kernel_shap_operator}
For $1 \leq i \leq d$, let $\kr_i^*: F \to F$ be defined as the operator with $\kr_i^*f(x) = \kr_i(\mu^*, f, x)$.
\end{definition}

Our main idea will be to rewrite $\kr_i^*$ as a linear combination of value operators. This will allow us to apply the same techniques that we've used to prove Theorems \ref{main_theorem} and Theorem \ref{theorem:robust_distribution_bound}. 

To do so, we will need several bits of algebra, which we break into the following lemmas. We will also let $b_x$ denote the value of $b$ that corresponds to $x$, as $b$ was defined with respect to a point $x$ (Definition \ref{defn:kern_limit_object}). 

\begin{lemma}\label{lem:pdqpdqpdqpdq}
$\one^tM^{-1}b_x = \sum_{0 < |S| < d} \frac{|S|\pi(S)}{p + dq}v_S^*f(x)$.
\end{lemma}

\begin{proof}
We just brute force it out. We have
\begin{equation*}
\begin{split}
\one^tM^{-1}b_x &= ((M^{-1})^t\one)^t b_x \\
&= \frac{1}{p + dq}\one^t b_x \\
&= \frac{1}{p + dq} \one^t \sum_{0 < |S| < d} \one_S \pi(S)v_S(\mu^*, f, x) \\
&= \sum_{0 < |S| < d} \frac{|S|\pi(S)}{p + dq}v_S^*f(x),
\end{split}
\end{equation*}
where in the middle step we simply used the fact that $\one$ is clearly an eigenvector of $M$ with eigenvalue $p+dq$, along with the the observation that $\one^t\one_S = |S|$.  
\end{proof}

\begin{lemma}\label{lem:alpha_beta}
For $1 \leq i \leq d$, $(M^{-1}b_x)_i = \sum_{S \subseteq [d] \setminus \{i\}} \alpha_{S \cup \{i\}} v_{S \cup \{i\}}^*f(x) - \beta_S v_S^*f(x)$, where $$\alpha_{S \cup \{i\}} = \left(\frac{1}{p} - \frac{q|S \cup \{i\}|}{p(p+dq)}\right) \pi(S \cup \{i\}) \text{ and }\beta_S = \left( \frac{q|S |}{p(p+dq)}\right) .$$ 
\end{lemma}

\begin{proof}
With some more brute force, 
\begin{equation*}
\begin{split}
M^{-1}b_x 
&= (pI_d + qJ_d)^{-1}  \sum_{S \subseteq [d]} \one_S \pi(S)v_S(\mu^*, f, x) \\
&= \left(\frac{1}{p}I_d - \frac{q}{p(p+qd)}J_d\right)  \sum_{S \subseteq [d]} \one_S \pi(S)v_S(\mu^*, f, x) \\
&= \left(\frac{1}{p}I_d - \frac{q}{p(p+qd)}J_d\right)  \sum_{S \subseteq [d] \setminus \{i\}} \left(\one_{S \cup \{i\}} \pi(S \cup \{i\})v_{S \cup \{i\}}^*f(x) +  \one_S \pi(S)v_S^*f(x)\right).
\end{split}
\end{equation*}
Now, observing that $(\one_{S \cup \{i\}})_i = 1$ and $(\one_S)_i = 0$ if $i \notin S$, we see that 
\begin{equation*}
\begin{split}
\left(\frac{1}{p}I_d - \frac{q}{p(p+qd)}J_d\right)&\sum_{S \subseteq [d] \setminus \{i\}} \one_{S \cup \{i\}} \pi(S \cup \{i\})v_{S \cup \{i\}}^*f(x)\\
&=  \sum_{S \subseteq [d] \setminus \{i\}}\left(\frac{1}{p} - \frac{q|S \cup \{i\}|}{p(p+dq)}\right) \pi(S \cup \{i\})v_{S \cup \{i\}}^*f(x)\big)
\end{split}
\end{equation*}
and
\begin{equation*}
\begin{split}
\left(\frac{1}{p}I_d - \frac{q}{p(p+qd)}J_d\right)&\sum_{S \subseteq [d] \setminus \{i\}} \one_{S } \pi(S )v_{S}^*f(x)\\
&=  \sum_{S \subseteq [d] \setminus \{i\}}\left(- \frac{q|S |}{p(p+dq)}\right) \pi(S)v_{S}^*f(x)\big)
\end{split}
\end{equation*}
\end{proof}

\begin{lemma}[Expressing $\kr_i$ using value operators]\label{def:kernel_shap_operator_algebra}
Let $1 \leq i \leq d$. There exist real numbers $\iota_S$ for all $S \subseteq [d]$ such that the following hold:
\begin{enumerate}
	\item $\kr_i^* = \sum_{S \subseteq [d]} \iota_Sv_S^*$. 
	\item $\iota_S \geq 0$ for all $S$ where $i \in S$. 
	\item $\iota_{[d]} = \frac{1}{d}$. 
\end{enumerate}
\end{lemma}

\begin{proof}
We explicitly compute $\kr_i^*f(x) = \kr_i(\mu^*, f, x)$. When possible we simplify using the two lemmas above. Using that $v_{\emptyset}^*f=0$ and that $v_{[d]}^*$ is the identity map, we have
\begin{equation*}
\begin{split}
\kr_i^*f(x) &= \left(M^{-1}\left(b_x - \one \frac{\one^tM^{-1}b_x - f(x) + v_\emptyset(\mu^*, f, x)}{\one^tM^{-1}\one} \right)\right)_i \\
&= \left(M^{-1}b_x - \frac{M^{-1} \one}{\one^t M^{-1} \one} \left(\one^tM^{-1}b_x - f(x)\right)\right)_i \\
&= \left(M^{-1}b_x - \frac{\one}{d} \left(\one^tM^{-1}b_x - v_{[d]}^*f(x)\right)\right)_i \\
&= \sum_{S \subseteq [d] \setminus \{i\}} \alpha_{S \cup \{i\}} v_{S \cup \{i\}}^*f(x) - \beta_S v_S^*f(x) + \frac{1}{d}v_{[d]}^*f(x) - \sum_{S \subseteq [d]} \frac{|S|\pi(S)}{d(p +dq)}v_S^*f(x),
\end{split}
\end{equation*}
with the last step coming from substituting Lemmas \ref{lem:pdqpdqpdqpdq} and \ref{lem:alpha_beta}. Regrouping terms, we see that $$\iota_S = \begin{cases}\alpha_S - \frac{|S|\pi(S)}{d(p+dq)} & i \in S, |S| < d \\  -\beta_S - \frac{|S|\pi(S)}{d(p+dq)} & i \notin S, |S| < d \\ \frac{1}{d} & S = [d] \end{cases}.$$ Properties 1 and 3 of the lemma clearly hold, so all that is left is Property 2. To this end, we see that for $i \in S$,
\begin{equation*}
\begin{split}
\iota_S &= \alpha_S - \frac{|S|\pi(S)}{d(p+dq)} \\
&= \left(\frac{1}{p} - \frac{q|S|}{p(p+dq)}\right) \pi(S) - \frac{|S|\pi(S)}{d(p+dq)} \\
&= \pi(S) \left(\frac{1}{p} - \frac{q|S|}{p(p+dq)}- \frac{|S|}{d(p + dq)}\right) \\
&\geq \pi(S) \left(\frac{1}{p} - \frac{qd}{p(p+dq)}- \frac{|S|}{d(p + dq)}\right) \\
&= \pi(S) \left(\frac{1}{p}\left(1 - \frac{qd}{(p+dq)}\right)- \frac{|S|}{d(p + dq)}\right) \\
&= \pi(S) \left(\frac{1}{p}\left(\frac{p}{p+dq}\right) -  \frac{|S|}{d(p + dq)}\right) \\
&\geq \pi(S) \left(\frac{1}{p + dq} - \frac{d}{d(p+dq)}\right) \\
&= 0. 
\end{split}
\end{equation*}
\end{proof}

We are now prepared to prove Theorem \ref{thm:kernel_shap}. Our proof CLOSELY follows the proof of Theorem~\ref{theorem:robust_distribution_bound}.

\begin{proof}[Theorem \ref{thm:kernel_shap}]
We first reduce the empirical KernelSHAP values to the distribution KernelSHAP values by using the error term (Definition \ref{defn:error_term}). We have
\begin{equation*}
\begin{split}
\Ev_{x \sim \mu^*} |\kr_i(\mu^*, f, x)| &\leq \eta\left(X^*, \mu^*, f\right) + \frac{1}{n}\sum_{x^{(j)} \in X^*} |\kr_i(X^*, f, x^{(j)})| \\
&\leq \eta\left(X^*, \mu^*, f\right) + \epsilon,
\end{split}
\end{equation*}
with the latter holding from the Theorem statement. 

We now use the same argument we did for proving Theorem \ref{theorem:robust_distribution_bound}, the only difference is that the operators we use to express the KernelSHAP operator differ from the ones used for the true SHAP values. Nevertheless, we will see that the same properties hold. In particular, Lemma \ref{def:kernel_shap_operator_algebra} implies that
\begin{equation*}
\begin{split}
\kr_i^* &= \sum_{S \subseteq [d]} \iota_S v_S^* \\ 
&= \sum_{S \subseteq [d] \setminus \{i\}} \iota_{S \cup \{i\}}v_{S \cup \{i\}}^* - \sum_{S \subseteq [d] \setminus \{i\}} \iota_{S}v_{S}^* \\
&= C_i^* - D_i^*.
\end{split}
\end{equation*}
The key idea is that $C_i^*$ and $D_i^*$ fulfill \textit{precisely} the same qualities that $A_i^*$ and $B_i^*$ did. That is, 
\begin{enumerate}
	\item $C_i^*(F_f^*) \subseteq F_f^*$,
	\item $C_i^*(W) \subseteq W$,
	\item $D_i^*(F_f^*) \subseteq W$,
	\item $(C_i^*)^{-1}\left(\{0\}\right) = \{0\}$, 
\end{enumerate}
where $F_f^*$ and $W$ are just as in the proof of Theorem \ref{theorem:robust_distribution_bound}. To see this, simply observe that $C_i^*$ and $D_i^*$ are linear combinations of the \textit{same value operators} as $A_i^*$ and $B_i^*$. Thus the arguments used to prove the properties of $A_i^*$ and $B_i^*$ (i.e. Lemmas \ref{lemma:key_properties}, \ref{appxlemma:eigenvalues_of_A_i} and \ref{lemma:localized_subspace}) equally apply as the only facts we \textit{ever} used about the coefficients of these linear combinations was that the coefficients in the expression for $A_i^*$ were nonnegative and also equal to $\frac{1}{d}$ in the case of $v_{[d]}^*$. Thus, the proof to this Theorem immediately follows by simply replacing $A_i^*$ and $B_i^*$ from the proof of Theorem \ref{theorem:robust_distribution_bound} with $C_i^*$ and $D_i^*$.
\end{proof}

\section{Proofs from Section \ref{sec:lie_algebra}}

\subsection{Proof of Lemma \ref{lem:value_operators_awesome}}\label{app:this_is_the_last_proof_please}

\begin{proof}[Lemma~\ref{lem:value_operators_awesome}]
Let $f \in F_T$ be a $T$-determined function. By definition, 
\begin{equation}\label{eqn:main}
(v_Sf)(x) = \Ev_{X \sim \mu}f\left(x_S, X_{S^c}\right) = \Ev_{X \sim \mu} f\left(x_{S \cap T}, x_{S \setminus T}, X_{T \setminus S}, X_{(S \cup T)^c}\right).
\end{equation} 
To prove Property~1 of the lemma, let $x, x' \in \supp(\mu^*)$ satisfy $x_{S \cap T} = x_{S \cap T}'$. Because $f$ is $T$-determined, changing coordinates of an input point outside $T$ does not effect its function value. Thus changing $x_{S \setminus T}$ to $x_{S \setminus T}'$ in Equation \ref{eqn:main} and noting $x_{S \cap T} = x_{S \cap T}'$ gives us 
\begin{equation*}
\begin{split}
(v_Sf)(x) &= \Ev_{X \sim \mu} f\left(x_{S \cap T}, x_{S \setminus T}, X_{T \setminus S}, X_{(S \cup T)^c}\right) \\
&= \Ev_{X \sim \mu} f\left(x_{S \cap T}', x_{S \setminus T}', X_{T \setminus S}, X_{(S \cup T)^c}\right) = (v_Sf)(x').
\end{split}
\end{equation*}
Since $x, x'$ were arbitrary, this implies $v_Sf$ is $S \cap T$-determined.
To prove Part 2, we further simplify Equation~\ref{eqn:main} by noting that $T \subseteq S$ to get $$(v_Sf)(x) = \Ev_{X \sim \mu} f\left(x_{T}, x_{S \setminus T}, X_{S^c}\right).$$ Because $f \in F_T$, for any choice of $X$ we have $f(x_T, x_{S \setminus T}, X_{S^c}) = f(x_T, x_{S \setminus T}, x_{S^c}) = f(x)$. Substituting this into the expectation above implies $(v_S f)(x) = f(x)$, completing the proof. 
\end{proof}

\subsection{Proof of Lemma \ref{lem:shap_lie_solvable}}\label{proof:lem:shap_lie_solvable}

We begin with a useful technical lemma that will help us connect pure operators (Definition~\ref{defn:pure_operator}) to derived Lie sub-algebras of $\sg$. 

\begin{lemma}\label{lem:closed_under_derivation}
Let $W$ be a set of linear operators that is closed under derivations. Then $$span(W) = \left\{\sum_{i = 1}^n \lambda_i w_i: w_1, \dots, w_n \in W, \lambda_1, \dots, \lambda_n \in \C, n\in \N\right\},$$ is a Lie algebra. 
\end{lemma}

\begin{proof}[Lemma \ref{lem:closed_under_derivation}]
It suffices to show that $span(W)$ is closed under derivations as it is by definition a vector space. To do so, we appeal to the linearity of derivations. Observe that for any $a, b, c \in W$ and $\lambda \in \C$ it holds that 
\begin{enumerate}
	\item $[a + b, c] = (a+b)c - c(a+b) = (ac - ca) + (bc - cb) = [a, c] + [b, c]$.
	\item $[a, b+ c] = a(b+c) - (b+c)a = (ab - ba) + (ac - ca) = [a, b] + [a, c]$.
	\item $[a, \lambda b] = a(\lambda b) - \lambda (ba) = \lambda[a, b]$.
	\item $[\lambda a, b] = (\lambda a)b - b (\lambda a) = \lambda[a, b]$. 
\end{enumerate}
Applying these properties, we see that for $\sum \lambda_i w_i, \sum \lambda_i'w_i' \in span(W)$, $$\left[\sum_{i=1}^n \lambda_i w_i, \sum_{i = 1}^{n'} \lambda_i' w_i' \right] = \sum_{i=1}^n \sum_{j = 1}^{n'} \lambda_i\lambda_j' [w_i, w_j'].$$ The latter sum is a linear combination of elements from $W$ as $[w_i, w_j'] \in W$ for all $i, j$. It follows that $\left[\sum_{i=1}^n \lambda_i w_i, \sum_{i = 1}^{n'} \lambda_i' w_i' \right] \in span(W)$ as desired. 
\end{proof}

We first show that pure operators (Definition \ref{defn:pure_operator}) are sufficient for constructing linear algebraic bases of all Lie sub-algebras $\sg^{(i)}$ in the derived series of $\sg$.

\begin{lemma}\label{lem:shapley_basis_pure}
Let $V^{(0)} = V$ be the set of all pure operators. For $i \geq 1$, let $V^{(i)} = \{[v, w]: v, w \in V^{(i-1)}\}$. Then for all $i \geq 0$, $V^{(i)}$ is closed under derivations and satisfies $span(V^{(i)}) = \sg^{(i)}$. 
\end{lemma}

\begin{proof}[Lemma \ref{lem:shapley_basis_pure}]
We proceed by induction on $i$. For the base case, $V^{(0)}$ is by definition the minimal set closed under derivations that contains $\{v_S: S \subseteq [d]\}$. Thus Lemma~\ref{lem:closed_under_derivation} implies that $span(V^{(0)})$ is a Lie algebra containing $\{v_S: S \subseteq [d]\}$ which implies that $\sg^{(0)} \subseteq span(V^{(0)})$. On the other hand $V^{(0)}$ is a clear subset of $\sg^{(0)}$ as $V^{(0)}$ is the minimal set containing $\{v_S: S \subseteq [d]\}$ that is closed under derivations. Thus, since $\sg^{(0)}$ is a Lie algebra and thus a vector space, we have $span(V^{(0)}) \subseteq \sg^{(0)}$ which implies equality. 

Next, suppose the inductive hypothesis holds for $i-1$. We first show that $V^{(i)}$ is closed under derivations. Let $[v, w]$ and $[v', w']$ be two elements in $V^{(i)}$ with $v, w, v', w' \in V^{(i-1)}$. Since $V^{(i-1)}$ is closed under derivations (inductive hypothesis), it follows that $[v, w]$ and $[v', w']$ are themselves elements of $V^{(i-1)}$. The definition of $V^{(i)}$ implies their derivation $[[v, w], [v', w']]$ is an element of $V^{(i)}$, which proves closure.

Next, by the definition of $V^{(i)}$, we see that $$span(V^{(i)}) \subseteq [span(V^{(i-1)}), span(V^{(i-1)})] = [\sg^{(i-1)}, \sg^{(i-1)}] = \sg^{(i)}.$$ In the other direction, Lemma~\ref{lem:closed_under_derivation} implies that $span(V^{(i)})$ is itself a Lie algebra (as $V^{(i)}$ is closed under derivations). Since $\sg^{(i)}$ is defined as the smallest Lie algebra that contains $[v, w]$ for $v, w \in \sg^{(i-1)}$, it suffices to show that $span(V^{(i)})$ contains this as well.

To this end let $v, w \in \sg^{(i-1)}$. Applying the inductive hypothesis, we express them in their basis from $V^{(i-1)}$ by setting $v = \sum_{j=1}^n \lambda_j v_j$ and $w = \sum_{k=1}^m \mu_k w_k$ for $v_1, \dots, v_n, w_1, \dots, w_m \in V^{(i-1)}$, we see that $$[v, w] = \sum_j \sum_k \lambda_j \mu_k [v_j, w_k].$$ This is clearly in $span(V^{(i)})$ as each element $[v_j, w_k]$ is in $V^{(i)}$ by definition. This completes the proof. 
\end{proof}

Next, we show how the sets $V^{(i)}$ interact with the subsets constructed in Lemma~\ref{lem:images_of_pure_operators}.

\begin{lemma}[Size of Associated Subsets]\label{lem:things_are_small}
$\alpha(v)$ satisfies the following two properties. 
\begin{enumerate}[noitemsep]
	\item If $\alpha(v_1) = \alpha(v_2)$, then $[v_1, v_2] = 0$.
	\item For all $v \in V^{(i)}$, $|\alpha(v)| \leq \max \left(d-i, 0 \right)$. 
\end{enumerate}

\end{lemma}

\begin{proof}[Lemma \ref{lem:things_are_small}]
We begin with the first claim. Observe that for any $f \in F$, Lemma~\ref{lem:images_of_pure_operators} implies $v_2f \in F_{\alpha(v_2)}$. However, it also implies $v_1(F_{\alpha(v_1)}) = v_1(F_{\alpha(v_2)}) = 0$. Thus $v_1v_2f = 0$. Similarly $v_2v_1f = 0$. It follows that $[v_1, v_2]f = 0$ implying $[v_1, v_2] = 0$

We prove the second claim by induction on $i$. The base case holds because all subsets have size at most $d$. Next, suppose the inductive hypothesis holds for $(i- 1)$.

Let $v \in V^{(i)}$. By definition, there exist $v_1, v_2 \in V^{(i-1)}$ for which $v = [v_1, v_2]$. Lemma \ref{lem:images_of_pure_operators} states that $\alpha(v) \subseteq \alpha(v_1) \cap \alpha(v_2)$. This gives us two cases. 

First, if $\alpha(v_1) \neq \alpha(v_2)$, then $$|\alpha(v)| \leq |\alpha(v_1) \cap \alpha(v_2)| < \max\left(|\alpha(v_1)|, |\alpha(v_2)| \right) \leq \max(d-i+1, 0).$$ The strictness of this inequality implies that $|\alpha(v)| \leq \max(d-i, -1) = d-i$ which implies the inductive hypothesis holds for $i$.

Second, if $\alpha(v_1) = \alpha(v_2)$, then the first claim of the lemma implies $[v_1, v_2] = 0$ which immediately implies $\alpha([v_1, v_2]) \leq \max(d-i, 0)$, completing the inductive hypothesis.  
\end{proof}

We are finally prepared to prove Lemma \ref{lem:shap_lie_solvable}.

\begin{proof}[Lemma \ref{lem:shap_lie_solvable}]
By Lemma \ref{lem:things_are_small}, $\alpha(v) = \emptyset$ for all $v \in V^{(d)}$. Next, let $v \in V^{(d+1)}$. By definition, $v = [v_1, v_2]$ for $v_1, v_2 \in V^{(d)}$. Since $\alpha(v_1) = \alpha(v_2)$, Lemma \ref{lem:things_are_small} implies $[v_1, v_2] = 0$ which means $v = 0$. Thus $V^{(d+1)}= 0$. Since $V^{(d+1)}$ spans $\sg^{(d+1)}$ as a vector space (Lemma \ref{lem:shapley_basis_pure}), it follows that $\sg^{(d+1)} = 0$ which means $\sg$ is solvable. 
\end{proof}

\subsection{Proof of Lemma \ref{lem:images_of_pure_operators}}\label{proof:lem:images_of_pure_operators}

We begin with a useful lemma that characterizes the intersection of determined function spaces. 

\begin{appxlemma}[intersection of determined spaces]\label{lem:det_space_intersection}
    For $S, T \subseteq [d]$, $F_S \cap F_T = F_{S \cap T}$.
\end{appxlemma}

\begin{proof}[Lemma \ref{lem:det_space_intersection}]
We first show that $F_{S \cap T} \subseteq F_S \cap F_T$. Let $f \in F_{S \cap T}$, and let $x, x'$ satisfy $x_S = x_S'$. Then $x_{S \cap T} =x_{S \cap T}'$ as well which implies $f(x_S) = f(x_S')$ by the definition of a determined function. Thus $f$ is $S$-determined meaning $f \in F_S$. We can similarly show $f \in F_T$.

Next, we show $F_S \cap F_T \subseteq F_{S \cap T}$. Let $f \in F_S \cap F_T$ and let $x, x'$ satisfy $x_{S \cap T} = x'_{S \cap T}$. Using $f \in F_S$, $f \in F_T$, and $x_{S \cap T} = x_{S \cap T}'$, we get 
\begin{align*}
f(x) &= f\left(x_{S \cap T}, x_{S \setminus T}, x_{T \setminus S}, x_{S^c \cap T^c}\right)\\
&= f\left(x_{S \cap T}, x_{S \setminus T}, x_{T \setminus S}', x_{S^c \cap T^c}'\right) \tag{$f \in F_S$}\\
&= f\left(x_{S \cap T}, x_{S \setminus T}', x_{T \setminus S}', x_{S^c \cap T^c}'\right) \tag{$f \in F_T$}\\
&= f\left(x_{S \cap T}', x_{S \setminus T}', x_{T \setminus S}', x_{S^c \cap T^c}'\right) = f(x'), \tag{$x_{S \cap T} = x_{S \cap T}'$}
\end{align*}
which implies the result.
\end{proof}

We now prove Lemma~\ref{lem:images_of_pure_operators}.

\begin{proof}[Lemma~\ref{lem:images_of_pure_operators}]
We first modify the extend of the criteria for $\alpha(v)$ to also apply to operators in $V \setminus \sg^{(1)}$. We say that an operator $v \in V$ has a \textit{nice} subset $S$ if 
\begin{enumerate}
	\item $v(F_T) \subseteq F_{S \cap T}$ for all $T \subseteq [d]$.
	\item If $T \subseteq S$ and $f \in F_T$, then $vf = \begin{cases} f & v \notin \sg^{(1)} \\ 0 & v \in \sg^{(1)} \end{cases}.$
\end{enumerate}
We begin by showing that the set of nice subsets is closed under intersection. Suppose that $S, S'$ are nice with respect to operator $v$. Then applying Lemma \ref{lem:det_space_intersection}, we have
\begin{enumerate}
	\item $v(F_T) \subseteq \left(F_{S \cap T} \cap F_{S' \cap T}\right) = F_{S \cap S' \cap T}$. 
	\item If $T \subseteq S \cap S'$ then $T \subseteq S$. Since $S$ is nice, we have for all $f \in F_T$, $vf = \begin{cases} f & v \notin \sg^{(1)} \\ 0 & v \in \sg^{(1)} \end{cases}$. 
\end{enumerate}
We now define $\alpha(v)$ as the intersection of all nice subsets that $v$ has. By our previous observation, $\alpha(v)$ itself is nice with respect to $v$ and thus satisfies the first two properties of Lemma~\ref{lem:images_of_pure_operators}. 

To complete the proof, it suffices to show that $\alpha(v)$ is well defined for all $v \in V$, and that it also satisfies Property 3 of Lemma~\ref{lem:images_of_pure_operators}. To this end, let $V' \subseteq V$ denote the set of all operators in $V$ that have at least one nice subset. We claim that $V'$ is closed under derivations. To see this, let $v, w \in V'$. 

For any $T \subseteq [d]$, observe that because $\alpha(v), \alpha(w)$ are nice w.r.t. $v, w$, $$(vw)(F_T) \subseteq v(F_{\alpha(w) \cap T}) \subseteq F_{\alpha(v) \cap \alpha(w) \cap T},$$ $$(wv)(F_T) \subseteq w(F_{\alpha(v) \cap T}) \subseteq F_{\alpha(w) \cap \alpha(v) \cap T}.$$ Since $F_{\alpha(v) \cap \alpha(w) \cap T}$ is a vector space, it follows that $[v, w](F_T) \subseteq F_{\alpha(v) \cap \alpha(w) \cap T}$ thus showing that $\alpha(v) \cap \alpha(w)$ satisfies Property~1 of being a nice subset.

Next, let $T \subseteq \alpha(v) \cap \alpha(w)$. Let $\lambda_v = 1\left(v \notin \sg^{(1)}\right)$ and $\lambda_w = 1\left(w \notin \sg^{(1)}\right)$. Applying Property~2 of nice subsets to $\alpha(v), \alpha(w)$, we have that for any $f \in F_T$, $$[v, w]f = (vw - wv)f = v \lambda_wf - w\lambda_v f = \lambda_v\lambda_wf - \lambda_w\lambda_v f = 0.$$ Thus, $\alpha(v) \cap \alpha(w)$ is nice with respect to $[v, w]$. Moreover, by the definition of $\alpha$, this implies that $\alpha([v, w]) \subseteq \alpha(v) \cap \alpha(w)$. 

Having verified all three properties, all that is left to show is that $V' = V$. To do so, observe that $V'$ contains $v_S$ for all $S \subseteq [d]$ as $S$ is clearly a nice subset with respect to $v_S$ (Lemma~\ref{lem:value_operators_awesome}). Thus $V'$ is a set closed under derivations that contains $\{v_S: S \subseteq [d]\}$. The definition of $V$ implies that $V \subseteq V'$, and this implies equality as desired. 
\end{proof}

\section{Definitions and Theorems about Lie Algebras}\label{sec:lie_algebra_appendix_stuff}

\begin{appdfn}[Solvable Lie Algebra]\label{appdefn:solvable_lie_algebra}
For any Lie algebra, $\g$, define the following:
\begin{enumerate}[noitemsep]
	\item Its derivation $[\g, \g]$ is the Lie sub-algebra generated by $\{[v_1, v_2]: v_1, v_2 \in \g\}$.
	\item Its derived series is the sequence $\g^{(i)} = [\g^{(i-1)}, \g^{(i-1)}]$ with $\g^{(0)} = \sg$.
\end{enumerate}
Finally, $\g$ is solvable if there exists $n$ such that $\sg^{(n)} = 0$. 
\end{appdfn}

\begin{appdfn}[Representation of Lie Algebra]\label{defn:representation}
Let $\g$ be a Lie algebra and let $GL(W)$ denote the general linear algebra over some vector space $W$. A representation of $\g$ is a pair $(\rho,W)$ where $W$ is a vector space, and $\rho: \g \to GL(W)$ maps each element of $\g$ to a linear transformation over $W$ such that for all $v, w \in \g$,
\begin{enumerate}[noitemsep]
	\item $\rho(av + bw) = a\rho(v) + b\rho(w)$ for all $a, b \in \C$.
	\item $\rho([v, w]) = \rho(v)\rho(w) - \rho(w)\rho(v)$. 
\end{enumerate}
$(\rho, W)$ is said to be finite dimensional if $W$ is.
\end{appdfn}

\begin{apptheorem}[Lie's Theorem]\label{theorem:LIE_THEOREM}
Let $\mathfrak{g}$ be a solvable Lie algebra, and $(\rho, W)$ be a finite dimensional representation. Then there exists a basis of $W$ under which $\rho(g)$ is upper triangular for all $g \in \mathfrak{g}$.
\end{apptheorem}

\section{Constructing a Finite Dimensional Representation of \texorpdfstring{$\sg$}.}\label{app:finite_dimensional_stuff}

\begin{appdfn}[Extended Space of Value Operators]\label{defn:big_algebra}
We let $V^*$ denote the space of all operators that can be obtained through linear combinations and compositions of value operators. That is, $$V^* = span\left\{v_1v_2 \dots v_m: m \in \N, v_1, \dots, v_m \in \{v_S: S \subseteq [d]\} \right\}.$$
\end{appdfn}

\begin{appdfn}[Localized Subspace]\label{app:defn:localized_subspace}
For $f \in F$, define its localized subspace, $$F_f = V^*f = \{vf: v \in V^*\},$$ as the set of all functions that can be obtained from $f$ by applying an operator in $V^*$ to it. 
\end{appdfn}

We will use $F_f$ to construct a \textit{finite dimensional representation} of the Shapley Lie algebra. The definition of a finite dimensional representation can be found in Definition \ref{defn:representation} of Section \ref{sec:lie_algebra_appendix_stuff}.

\begin{appxlemma}[Local Representation of Shapley Lie Algebra]\label{lem:definition_local_representation}
Let $f \in F$ be any function. For all $v \in \sg$, let $\rho_f(v)$ be defined as the restriction of $v$ to $F_f$. Then $f \in F_f$, and $(\rho_f, F_f)$ is a well-defined finite dimensional representation of the Shapley Lie algebra $\sg$. 
\end{appxlemma}

\begin{proof}[Lemma \ref{lem:definition_local_representation}]
The fact that $f \in F_f$ is immediate: $v_{[d]}$ is precisely the identity operator and thus $V^*f$ contains $v_{[d]}f = f$. 

To prove $(\rho_f, F_f)$ is a well defined representation, it suffices to show that $V^*f$ is a $v$-invariant subspace of $F$ (meaning $v(V^*f) \subseteq V^*f$) for all $v \in \sg$ (the rest of the properties follow from basic properties of linear transformations). Let $\sg'$ denote the set of all $v \in \sg$ such that $V^*f$ is $v$-invariant. $\sg'$ clearly contains all value operators $v_S$ as $v_S(v_1 \dots v_m f)$ is itself a composition of value operators applied to $f$. $\sg'$ is also closed under linear combinations (because $V^*$ is a span of all compositions) and derivations (because it is closed under matrix multiplication). Thus, $\sg'$ forms a Lie algebra that contains all value operators, and thus $\sg'$ must contain $\sg$ (by the minimality of $\sg$). This proves $(\rho_f, F_f)$ is well-defined.

Next, we show $F_f$ is a finite dimensional vector space. The key idea is to show that if the element $v_1v_2 \dots v_m f$ satisfies $v_1 \in \{v_2, \dots, v_m\}$, then $v_1 \dots v_m f = v_2 \dots v_m f$. In short, $v_1$ has no effect after it has already been applied in a product. 

To prove this, let $v_i = v_{S_i}$ for $S_i \subseteq [d]$ and $v_1 \in \{v_2, \dots, v_m \}$. Then Property 1 of Lemma~\ref{lem:value_operators_awesome} implies $$v_2 \dots v_m f = v_{S_2} \dots v_{S_m} f \in F_{S_2 \cap \dots \cap S_m}.$$ Applying Property 2 of Lemma~\ref{lem:value_operators_awesome} with $S_2 \cap \dots \cap S_m \subseteq S_1$ implies that $v_{S_1}$ acts as the identity map, meaning that $$v_1\left(v_2 \dots v_mf\right) = v_{S_1} \left(v_2 \dots v_m f\right) = v_2 \dots v_mf.$$ Applying this claim repeatedly, it follows that any product $v_1 \dots v_m f$ can be reduced to one in which $m \leq 2^d$ as there are at most $2^d$ distinct subsets of $[d]$. It follows that  
\begin{equation*}
\begin{split}
V^* &= span\left\{v_1v_2 \dots v_m: m \in \N, v_1, \dots, v_m \in \{v_S: S \subseteq [d]\} \right\} \\
&= span\left\{v_1v_2 \dots v_m: m \in \N, v_1, \dots, v_m \in \{v_S: S \subseteq [d]\}, m \leq 2^d \right\}.
\end{split}
\end{equation*}
Because the latter space is a span of a finite number of elements, it follows that $V^*f$ itself is finite which completes the proof. 
\end{proof}

\begin{appxlemma}[Localized Subspace if preserved by $A_i$ and $B_i$]\label{lemma:localized_subspace}
Let $F_f$ be the localized subspace of~$f$. Then $f \in F_f$ and $A_i(F_f), B_i(F_f) \subseteq F_f$, where $A_i, B_i$ are as defined in Definition \ref{defn:shapley_operator}. 
\end{appxlemma}

\begin{proof}[Lemma \ref{lemma:localized_subspace}]
This immediately follows from Lemma \ref{lem:definition_local_representation} along with the fact that $A_i$ and $B_i$ are both in $\sg$ seeing as they are linear combinations of value operators. 
\end{proof}

%
%
%
%

\section{Finding Counterexamples with a Linear Program}\label{app:lin_prog_counterexample}

The key observation that makes it possible to use a Linear Program to construct counterexamples such as the one in Figure~\ref{fig:counterexample} is that the value function and therefore also the SHAP values themselves are linear in the function values.
This holds for both, the observational and interventional SHAP value function.\\
In order to exploit this linear nature of the SHAP values, we consider a piecewise constant function $f: \R^2 \to \R$, that only takes finitely many values. 
We achieve this by defining $f$ on a two dimensional $d_1 \times d_2$-grid, where $d_1$ and $d_2$ define the size of the grid for Features $1$ and $2$, respectively. Then we can represent $f$ as a $(d_1 \cdot d_2)$-dimensional vector and the SHAP values for Feature $1$ can be computed via matrix multiplication $\Phi_1 f$ where $\Phi_1$ is a $d_1\cdot d_2 \times d_1 \cdot d_2$-dimensional matrix.\\
Now as a constraint for the Linear Program we can simply set $\Phi_1 f  = 0$ which would force all the SHAP values on the extended support to be zero. 
If we only want to restrict the SHAP values inside the support to be zero, we can simply set the rows in $\Phi_1$ to zero, if they correspond to a grid cell, that lies outside the support. \\
Finally the objective of the Linear Program will be to find a function, that does depend on Feature $1$. This can be achieved by choosing two entries in the vector $f$ that correspond to two input points with the same $x_2$-value but different $x_1$-values. If they differ, the function depends on Feature $1$. Therefore maximizing their difference gives the desired results. However, many different approaches are possible here.\\
Figure~\ref{fig:counterexample} above shows one of these counterexamples that can be found using this Linear Program. Further counterexamples for smaller grids are displayed in Figure~\ref{fig:counterexample_3grid} and Figure~\ref{fig:counterexample_4grid}, where we set $d_1=d_2=3$ and $d_1=d_2=4$, respectively.\\
If we strengthen the constraint on the SHAP values to be zero on the full extended support, as described above, the Linear Program is not able to find a counterexample, which we would expect considering our theoretical results above. Figure~\ref{fig:positive_example_grid} displays one final example where the SHAP values of Feature $1$ are zero on the whole extended support and the function indeed does not depend on Feature $1$.

\begin{figure}[H]
    \centering
    \includegraphics[width=\textwidth]{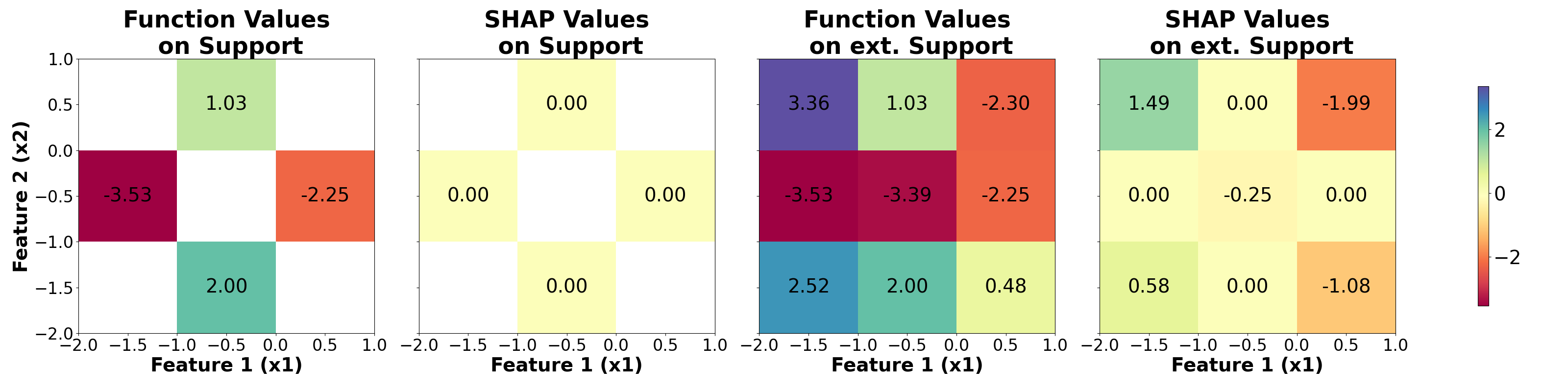}  
    \caption{
    \textbf{Example of a function where the aggregate SHAP value of Feature $1$ is $0$, yet the function depends on this feature on a $3 \times 3$-grid.} \textbf{(a):}  Function $f:\R^2 \to \R$, supported on only $4$ of the grid cells with the color depicting the function value. The function clearly depends on both Features $1$ and $2$. \textbf{(b):} Point-wise SHAP values $\phi_1(\mu, f,x)$ of Feature $1$ are constantly $0$ on the support. \textbf{(c) and (d):} Function and SHAP values on the extended support. Here the SHAP values are not constantly $0$ any more. 
    }
    \label{fig:counterexample_3grid} 
\end{figure}

\begin{figure}[H]
    \centering
    \includegraphics[width=\textwidth]{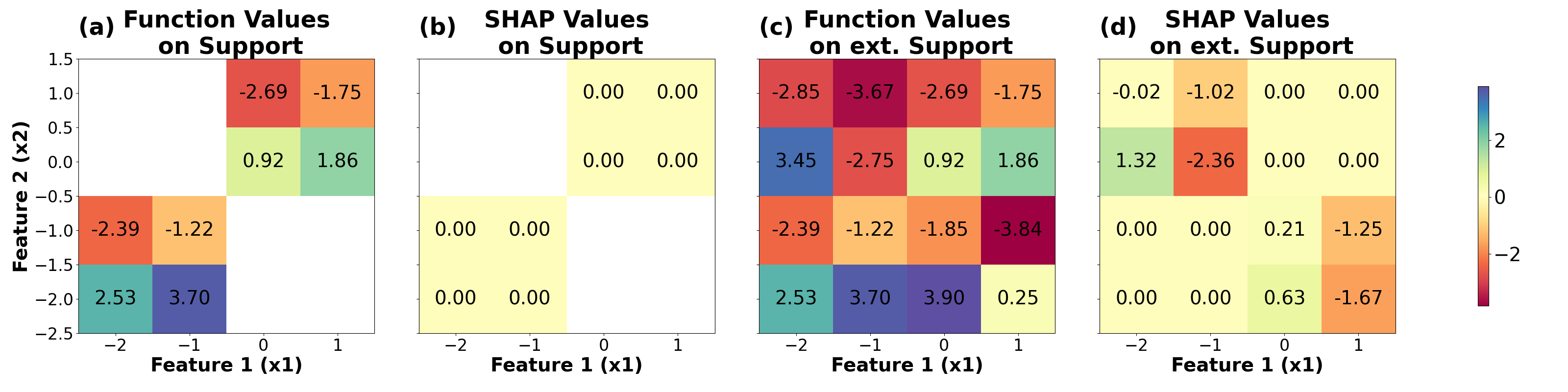}  
    \caption{
    \textbf{Example of a function where the aggregate SHAP value of Feature $1$ is $0$, yet the function depends on this feature on a $4 \times 4$-grid.} \textbf{(a):} Function $f:\R^2 \to \R$, supported on only $8$ of the grid cells with the color depicting the function value. The function clearly depends on both Features $1$ and $2$. \textbf{(b):} Point-wise SHAP values $\phi_1(\mu, f,x)$ of Feature $1$ are constantly $0$ on the support. \textbf{(c) and (d):} Function and SHAP values on the extended support. Here the SHAP values are not constantly $0$ any more. 
    }
    \label{fig:counterexample_4grid} 
\end{figure}

\begin{figure}[H]
    \centering
    \includegraphics[width=\textwidth]{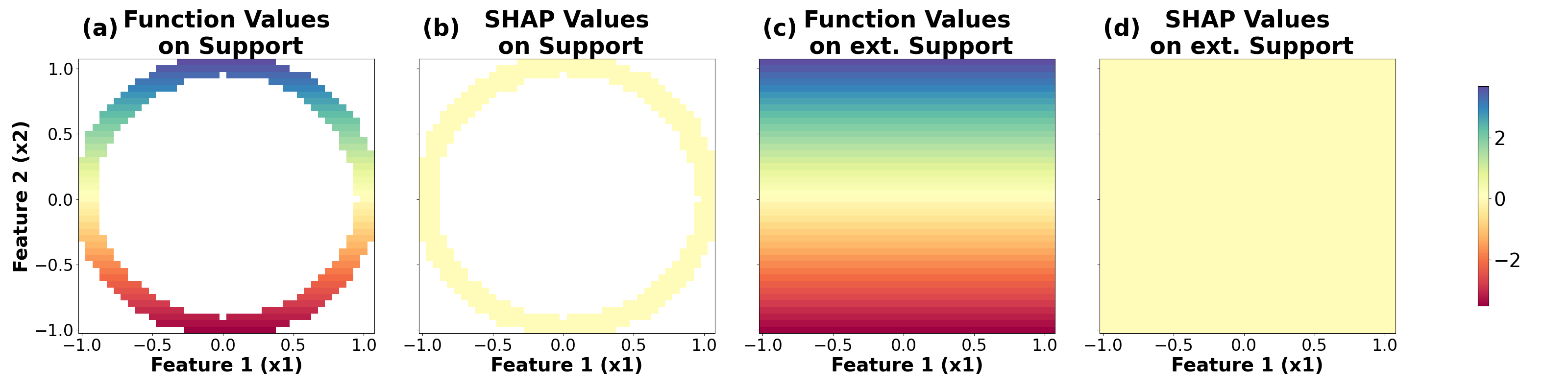}  
    \caption{
    \textbf{Example of a function where the aggregate SHAP value of Feature $1$ is $0$ on the whole extended support and the function does not depend on this feature.} \textbf{(a):} Function $f:\R^2 \to \R$, supported on a ring with the color depicting the function value. The function solely depends on Feature $2$. \textbf{(b):} Point-wise SHAP values $\phi_1(\mu, f,x)$ of Feature $1$ are constantly $0$ on the support. \textbf{(c) and (d):} Function and SHAP values on the extended support. Here the SHAP values are constantly $0$ on the extended support as well. 
    }
    \label{fig:positive_example_grid} 
\end{figure}



\section{Use cases of mean absolute SHAP in literature}\label{app:use_cases}
This section gives an overview over different use cases of the mean absolute SHAP value in science literature. Table~\ref{table:Applications_Science} holds an incomplete list of examples from recent years. In most applications the focus lies on the top features and while the possibility of doing feature selection based on the mean absolute SHAP value is often mentioned, e.g., by \citet{Appl_MAS:SharmaTimilsina2024}, scientists are careful in actually applying it. It is merely used to select features for further analysis and interpretation.

\setlength\LTleft{-1.5cm}
\begin{longtable}{|l|p{2.5cm}|p{10cm}|}
    \caption{\textbf{Collection of some exemplary use cases of the mean absolute SHAP value in literature.}}\label{table:Applications_Science} \\
    
    \hline \multicolumn{1}{|c|}{\textbf{Reference}} & \multicolumn{1}{c|}{\textbf{Scientific field}} & \multicolumn{1}{c|}{\textbf{Use of mean abs. SHAP}} \\ \hline 
    \endfirsthead
    
    \multicolumn{3}{c}%
    {{\bfseries \tablename\ \thetable{} -- continued from previous page}} \\
    \hline \multicolumn{1}{|c|}{\textbf{Reference}} & \multicolumn{1}{c|}{\textbf{Scientific field}} & \multicolumn{1}{c|}{\textbf{Use of mean abs. SHAP}} \\ \hline 
    \endhead
    
    \hline
    \endfoot
    
    \hline \hline
    \endlastfoot

    \citet{Appl_MAS:Greenwood2024}& Environmental Science & They investigate the influence of environmental and socioeconomic factors on the use of safely managed drinking water services. The features are grouped and the mean absolute SHAP value is calculated for each of the $5$ groups.\\ 
    \hline
    \citet{Appl_MAS:SharmaTimilsina2024}& Physical\linebreak Science &
    They use ML to predict the heating value of different types of waste and compute the mean absolute SHAP value to analyze the influence of the $8$ features with a focus on the most important features. \\ 
    \hline
    \citet{Appl_MAS:Delavaux2023}& Environmental Science &
    They want to identify drivers of non-native plant invasions in native ecosystems. Mean absolute SHAP values are interpreted as feature importance.\\ 
    \hline
    \citet{Appl_MAS:Bernard2023}& Medical \linebreak Science &
    They predict the physiological age based on biological values routinely assessed for diagnosis and treatment-monitoring.
    They use mean absolute SHAP values to identify the top-$20$ out of $48$ variables.\\ 
    \hline
    \citet{Appl_MAS:Ekanayake2022} & Physical \linebreak Science&
    They predict the compressive strength of concrete depending on its constituents and use the mean absolute SHAP values for interpretation of the $8$ features. They focus on both, top and bottom features.\\
    \hline
    \citet{Appl_MAS:Wang2022}& Environmental Science &
    They want to understand pollutant removal in wastewater treatment plants and use the mean absolute SHAP value to choose the top-$4$ out of $32$ features and take a deeper look into their influence.\\ 
    \hline
    \citet{Appl_MAS:Rane2022}& Medical \linebreak Science &
    They analyze the IMAGEN data set to predict, based on brain images, whether a person is going to misuse alcohol. 
    The features are considered most significant if they have mean abs SHAP value at least two times higher than the average mean abs SHAP value across all features.\\ 
    \hline
    \citet{Appl_MAS:Qiu2022}& Medical \linebreak Science &
    They develop a deep learning framework to classify different causes for dementia and differentiate them from Alzheimer's disease. The mean absolute SHAP values are used for interpretation with a focus on the top-$15$ features.\\ 
    \hline
    \citet{Appl_MAS:Chen2022}& Physical \linebreak Science &
    They classify proteins into self-assembling and partner-dependent proteins. The mean absolute SHAP values are used to interpret the influence of the features. All features are considered for the analysis with a focus on the top features.\\ 
    \hline
    \citet{Appl_MAS:Yang2022}& Physical \linebreak Science& 
    They design a machine learning implementation for the discovery
    of innovative polymers with ideal performance. The top-12 important molecular descriptors are identified using aggregate SHAP values.\\ 
\end{longtable}

\end{document}